\documentclass{article}

\usepackage{xr-hyper}

\usepackage{microtype}
\usepackage{graphicx}
\usepackage{booktabs}
\usepackage{multicol}
\usepackage{xfrac}
\usepackage{listings}
\usepackage{hyperref}

\usepackage[accepted]{icml2021}

\usepackage{amsthm}
\usepackage{cancel}
\usepackage{mathtools}
\usepackage{dsfont}
\usepackage{nicefrac}
\usepackage{subcaption}
\usepackage{wasysym}
\usepackage{multirow}
\usepackage[subtle]{savetrees}
\usepackage{flushend}

\newcommand{\smallmnistmodel}{2C2F}
\newcommand{\mnistmodel}{4C3F}
\newcommand{\tinyimagenetmodel}{8C2F}
\newcommand{\mnistmodelsmall}{2C2F}
\newcommand{\cifarmodel}{6C2F}
\newcommand{\na}{-}

\newtheorem{definition}{Definition}
\newtheorem{theorem}{Theorem}
\newtheorem{lemma}{Lemma}

\newcommand\classRel{\stackrel{_\bot}{=}}
\newcommand\margin[2][\epsilon]{\bar{#2}^{#1}}
\newcommand\proj[2]{\text{proj$\left( #1 \rightarrow #2 \right)$}}

\usepackage[OT1]{fontenc} 

\icmltitlerunning{Globally-Robust Neural Networks}

\begin{document}

\twocolumn[
    \icmltitle{
        Globally-Robust Neural Networks
    }



    \icmlsetsymbol{equal}{*}

    \begin{icmlauthorlist}
    \icmlauthor{Klas Leino}{cmu}
    \icmlauthor{Zifan Wang}{cmu}
    \icmlauthor{Matt Fredrikson}{cmu}
    \end{icmlauthorlist}

    \icmlaffiliation{cmu}{Carnegie Mellon University, Pittsburgh, Pennsylvania, USA}

    \icmlcorrespondingauthor{Klas Leino}{kleino@cs.cmu.edu}
    \icmlcorrespondingauthor{Matt Fredrikson}{mfredrik@cmu.edu}

    \icmlkeywords{Machine Learning, ICML}

    \vskip 0.3in
]



\printAffiliationsAndNotice{}  

\setcounter{footnote}{1}

\begin{abstract}

The threat of adversarial examples has motivated work on training \emph{certifiably robust} neural networks to facilitate efficient verification of \emph{local robustness} at inference time.
We formalize a notion of \emph{global robustness}, which captures the operational properties of on-line local robustness certification while yielding a natural learning objective for robust training.
We show that widely-used architectures can be easily adapted to this objective by incorporating efficient global Lipschitz bounds into the network, yielding certifiably-robust models \emph{by construction} that achieve \emph{state-of-the-art} verifiable accuracy.
Notably, this approach requires significantly less time and memory than recent certifiable training methods, and leads to negligible costs when certifying points on-line; for example, our evaluation shows that it is possible to train a large robust Tiny-Imagenet model in a matter of hours.
Our models effectively leverage inexpensive global Lipschitz bounds for real-time certification, despite prior suggestions that tighter local bounds are needed for good performance;
we posit this is possible because our models are specifically trained to achieve tighter global bounds.
Namely, we prove that the maximum achievable verifiable accuracy for a given dataset is not improved by using a local bound.
An implementation of our approach is available on GitHub\footnote{Code available at ~{\fontsize{7.5}{11}\selectfont \url{https://github.com/klasleino/gloro}}}.
\end{abstract}


\section{Introduction}\label{sec:intro}

We consider the problem of training neural networks that are robust to input perturbations with bounded $\ell_p$ norm.
Precisely, given an input point, $x$, network, $F$, and norm bound, $\epsilon$, this means that $F$ makes the same prediction on all points within the $\ell_p$-ball of radius $\epsilon$ centered at $x$.

This problem is significant as deep neural networks have been shown to be vulnerable to \emph{adversarial examples}~\cite{PapernotMJFCS16,SzegedyZSBEGF13}, wherein perturbations are chosen to deliberately cause misclassification.
While numerous heuristic solutions have been proposed to address this problem, these solutions are often shown to be ineffective by subsequent adaptive attacks~\cite{Carlini17Detected}.
Thus, this paper focuses on training methods that produce models whose robust predictions can be efficiently certified against adversarial perturbations~\cite{lee20local_margin,tsuzuku18margin,NEURIPS2018_358f9e7b}.

\begin{figure}[t]
\centering
\includegraphics[page=1, height=0.55\columnwidth]{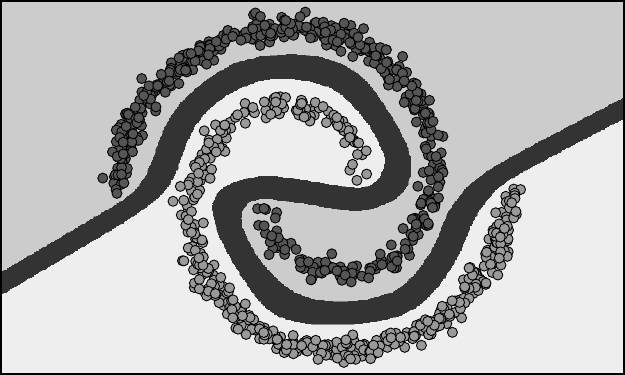}
\caption{
Illustration of global robustness. The model abstains from predicting on the margin between the classes (dark gray), which has width at least $\epsilon$.
}
\label{fig:margin_example}
\end{figure}

We begin by introducing a notion of \emph{global robustness} for classification models (Section~\ref{sec:global_rob}), which requires that classifiers maintain a separation of width at least $\epsilon$ (in feature space) between any pair of regions that are assigned different prediction labels.
This separation means that there are certain inputs on which a globally-robust classifier must refuse to give a prediction, instead signaling that a violation has occurred (see Figure~\ref{fig:margin_example}).
While requiring the model to abstain in some cases may at first appear to be a hindrance, we note that in operational terms this is no different than composing a model with a routine that only returns predictions when $\nicefrac{\epsilon}{2}$-local robustness can be certified.

While it is straightforward to construct a globally-robust model in this way via composition with a certification procedure, doing so with most current certification methods leads to severe penalties on performance or utility.
Most techniques for verifying local robustness are costly even on small models~\cite{Fischetti18MILP,fromherz20fgp,gehr2018ai,jordan19geocert,tjeng18MIP}, requiring several orders of magnitude more time than a typical forward pass of a network; on moderately-large CNNs, these techniques either time out after minutes or hours, or simply run out of memory.

One approach to certification that shows promise in this regard uses Lipschitz bounds to efficiently calculate the robustness region around a point~\cite{weng18fastlip,zhang18crown}.
In particular, when \emph{global} bounds are used with this approach, it is possible to implement the bound computation as a neural network of comparable size to the original (Section~\ref{sec:bounds}), making on-line certification nearly as efficient as inference.
Unfortunately, current training methods do not produce models with sufficiently small global bounds for this to succeed~\cite{weng18fastlip}.
Recent work~\cite{lee20local_margin} explored the possibility of training networks with sufficiently small \emph{local} bounds, but the training cost in time and memory remains prohibitive in many cases.

Surprisingly, we find that using global Lipschitz bounds for certification may not be as limiting as previously thought~\cite{huster2018limitations,yang20acc_rob_tradeoff}.
We show that for any set of points that can be robustly-classified using a local Lipschitz bound, there exists a model whose global bound implies the same robust classification (Theorem~\ref{thm:global_vs_local}).
This motivates a new approach to certifiable training that makes exclusive use of global bounds (Section~\ref{sec:gloro_nets}).
Namely, we construct a globally-robust model that incorporates a Lipschitz bound in its forward pass to define an additional ``robustness violation'' class, and use standard training methods to discourage violations while simultaneously encouraging accuracy.

Focusing on the case of deterministic guarantees against $\ell_2$-bounded perturbations, we show that this approach yields state-of-the-art verified-robust accuracy (VRA), while imposing little overhead during training and \emph{none} during certification.
For example, we find that we can achieve $63\%$ VRA with a large robustness radius of $\epsilon = 1.58$ on MNIST, surpassing all prior approaches by multiple percentage points.
We also achieve state-of-the-art VRA on CIFAR-10, and scale to larger applications such as Tiny-Imagenet (see Section~\ref{sec:eval}).

To summarize, we provide a method for training certifiably-robust neural networks that is simple, fast, capable of running with limited memory, and that yields state-of-the-art deterministic verified accuracy.
We prove that the potential of our approach is not hindered by its simplicity; rather, its simplicity is an asset---our empirical results demonstrate the many benefits it enjoys over more complicated methods.


\vspace{0.35em}
\section{Constructing Globally-Robust Networks}\label{sec:method}

In this section we present our method for constructing globally-robust networks, which we will refer to as \emph{GloRo Nets}.
We begin in Section~\ref{sec:global_rob} by formally introducing our notion of \emph{global robustness}, after briefly covering the essential background and notation.
We then show how to mathematically construct GloRo Nets in Section~\ref{sec:gloro_nets}, and prove that our construction is globally robust.

\subsection{Global Robustness}\label{sec:global_rob}

Let $f : \mathds{R}^n \rightarrow \mathds{R}^m$ be a neural network that categorizes points into $m$ different classes.
Let $F$ be the function representing the predictions of $f$, i.e., $F(x) = \text{argmax}_{i}\left\{f_i(x)\right\}$.

$F$ is said to be $\epsilon$-\emph{locally-robust} at point $x$ if it makes the same prediction on all points in the $\epsilon$-ball centered at $x$ (Definition~\ref{def:local_robustness}).

\begin{definition}
\label{def:local_robustness}
(Local Robustness) A model, $F$, is $\epsilon$-\emph{locally-robust} at point, $x$, with respect to norm, $||\cdot||$, if~~$\forall$ $x'$,
$$
||x - x'|| \leq \epsilon \ \ \Longrightarrow \ \ F(x) = F(x').
$$
\end{definition}
Most work on robustness verification has focused on this local robustness property;
in this work, we present a natural notion of \emph{global robustness}, which captures the operational properties of on-line local robustness certification.

Clearly, local robustness cannot be simultaneously satisfied at every point---unless the model is entirely degenerate, there will always exist points that are arbitrarily close to a decision boundary.
Instead, we will introduce a global robustness definition that can be satisfied even on models with non-trivial behavior by using an additional class, $\bot$, that signals that a point cannot be certified as globally robust.
At a high level, we can think of separating each of the classes with a margin of width at least $\epsilon$ in which the model always predicts $\bot$.
In order to satisfy global robustness, we require that no two points at distance $\epsilon$ from one another are labeled with different non-$\bot$ classes.

More formally, let us define the following relation ($\classRel$): we will say that $c_1 \classRel c_2$ if $c_1 = \bot$ or $c_2 = \bot$ or $c_1 = c_2$.
Using this relation, we provide our formal notion of global robustness in Definition~\ref{def:global_robustness}.

\begin{definition}
\label{def:global_robustness}
(Global Robustness) A model, $F$, is $\epsilon$-\emph{globally-robust}, with respect to norm, $||\cdot||$, if~~$\forall$ $x_1$, $x_2$,
$$
||x_1 - x_2|| \leq \epsilon \ \ \Longrightarrow \ \ F(x_1) \classRel F(x_2).
$$
\end{definition}

An illustration of global robustness is shown in Figure~\ref{fig:margin_example}.
While global robustness can clearly be trivially satisfied by labeling all points as $\bot$, we note that the objective of robust training is typically to achieve high robustness \emph{and} accuracy (i.e., VRA), thus ideally only points off the data manifold are labeled $\bot$, as illustrated in Figure~\ref{fig:margin_example}.

\subsection{Certified Globally-Robust Networks}\label{sec:gloro_nets}

Because of the threat posed by adversarial examples, and the elusiveness of such attacks against heuristic defenses~\cite{Carlini17Detected}, there has been a volume of previous work seeking to verify local robustness on specific points of interest.
In this work, we shift our focus to global robustness directly, resulting in a method for producing models that make predictions that are verifiably robust \emph{by construction}.

Intuitively, we aim to instrument a model with an extra output, $\bot$, that labels a point as ``not locally-robust,'' such that the instrumented model  predicts a non-$\bot$ class \emph{only if the point is locally-robust} (with respect to the original model).
At a high level, we do this by ensuring that in order to avoid predicting $\bot$, the maximum output of $f$ must surpass the other outputs by a sufficient margin.
While this margin is measured in the output space, we can ensure it is sufficiently large to ensure local robustness by relating the output space to the input space via an upper bound on the model's Lipschitz constant.

Suppose that $K_i$ is an upper bound on the Lipschitz constant for $f_i$. I.e., for all $x_1$, $x_2$, Equation~\ref{eq:lipschitz_def} holds.
Intuitively, $K_i$ bounds the largest possible change in the logit output for class $i$ per unit change in the model's input.
\begin{equation}
\label{eq:lipschitz_def}
\frac{|f_i(x_1) - f_i(x_2)|}{||x_1 - x_2||} \leq K_i
\end{equation}
Let $y = f(x)$, and let $j = F(x)$, i.e., the class predicted on point $x$. Let $y_\bot = \max_{i\neq j}\left\{y_i + (K_i + K_j)\epsilon\right\}$.
Intuitively, $y_\bot$ captures the value that the class that is most competitive with the chosen class would take under the worst-case change to $x$ within an $\epsilon$-ball.
Figure~\ref{fig:construction_illustration} provides an illustration of this intuition.

We then define the instrumented model, or GloRo Net, $\margin f$, as follows: $\margin f_i(x) ::= y_i$ $\forall i\in[m]$ and $\margin f_\bot(x) ::= y_\bot$; that is, $\margin f$ concatenates $y_\bot$ with the output of $f$.

\begin{figure}
\centering
\resizebox{\columnwidth}{!}{%
\includegraphics{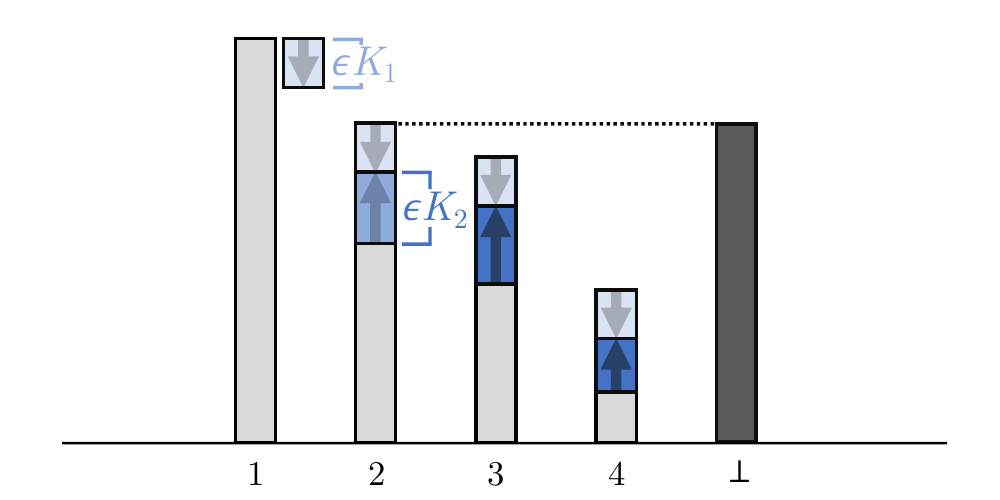}}
\caption{
	Illustration of calculating the $\bot$ logit.
	Note that $\epsilon K_i$ provides a bound on changes to logit $i$ within an $\epsilon$-ball.
	The $\bot$ logit is chosen to account for the predicted class \emph{decreasing} by the maximum amount and each other class \emph{increasing} by the maximum amount.
	If the $\bot$ logit does not surpass that of the predicted class, then no class can overtake the predicted class within an $\epsilon$-ball (Theorem~\ref{thm:local_robustness}).
}
\label{fig:construction_illustration}
\end{figure}

We show that the predictions of this GloRo Net, $\margin F$, can be used to certify the predictions of the instrumented model, $F$:
whenever $\margin F$ predicts a class that is not $\bot$, the prediction coincides with the prediction of $F$, and $F$ is guaranteed to be locally robust at that point (Theorem~\ref{thm:local_robustness}).

\begin{theorem}
\label{thm:local_robustness}
If $\margin F(x) \neq \bot$, then $\margin F(x) = F(x)$ and $F$ is $\epsilon$-locally-robust at $x$.
\end{theorem}

The proof of Theorem~\ref{thm:local_robustness} is given in Appendix~\ref{proof:local_robustness}.

Note that in this formulation, we assume that the predicted class, $j$, will decrease by the maximum amount within the $\epsilon$-ball, while all other classes increase by their respective maximum amounts. 
This is a conservative assumption that guarantees local robustness;
however, in practice, we can dispose of this assumption by instead calculating the Lipschitz constant of the margin by which the logit of the predicted class surpasses the other logits, i.e., the Lipschitz constant of $y_j - y_i$ for $i\neq j$.
The details of this tighter variant are presented in Appendix~\ref{proof:tighter_bound}, along with the corresponding correctness proof.

Notice that the GloRo Net, $\margin F$, will always predict $\bot$ on points that lie directly on the decision boundary of $F$.
Moreover, any point that is within $\epsilon$ of the decision boundary will also be labeled as $\bot$ by $\margin F$.
From this, it is perhaps clear that GloRo Nets achieve global robustness (Theorem~\ref{thm:global_robustness}).

\begin{theorem}
\label{thm:global_robustness}
$\margin[\nicefrac{\epsilon}{2}]{F}$ is $\epsilon$-globally-robust.
\end{theorem}

The proof of Theorem~\ref{thm:global_robustness} is given in Appendix~\ref{proof:global_robustness}.

\section{Revisiting the Global Lipschitz Constant}\label{sec:local_bounds}

The global Lipschitz constant gives a bound on the maximum rate of change in the network's output over the entire input space.
For the purpose of certifying robustness, it suffices to bound the maximum rate of change in the network's output over any pair of points \emph{within the $\epsilon$-ball} centered at the point being certified, i.e., the \emph{local} Lipschitz constant.
Recent work has explored methods for obtaining upper bounds on the local Lipschitz constant~\cite{weng18fastlip,zhang18crown,lee20local_margin};
the construction of GloRo Nets given in Section~\ref{sec:method} remains correct whether $K$ represents a global or a local Lipschitz constant.

The advantage to using a local bound is, of course, that we may expect tighter bounds; after all, the local Lipschitz constant is no larger than the global Lipschitz constant.
However, using a local bound also has its drawbacks.
First, a local bound is typically more expensive to compute.
In particular, a local bound always requires more memory, as each instance has its own bound, hence the required memory grows with the batch size.
This in turn reduces the amount of parallelism that can be exploited when using
a local bound, reducing the model's throughput.

Furthermore, because the local Lipschitz constant is different for every point, it must be computed every time the network sees a new point.
By contrast, the global bound can be computed in advance, meaning that verification via the global bound is essentially free.
This makes the global bound advantageous, assuming that it can be effectively leveraged for verification.

It may seem initially that a local bound would have greater prospects for successful certification.
First, \emph{local} Lipschitzness is sufficient for robustly classifying well-separated data~\cite{yang20acc_rob_tradeoff}; that is, global Lipschitzness is not necessary.
Meanwhile, global bounds on typical networks have been found to be prohibitively large~\cite{weng18fastlip}, while local bounds on in-distribution points may tend to be smaller on the same networks.
However, the potential disadvantages of a global bound become less clear if the model is specifically trained to have a small global Lipschitz constant.

For example, GloRo Nets that use a global Lipschitz constant will be penalized for incorrect predictions if the global Lipschitz constant is not sufficiently small to verify its predictions;
therefore, the loss actively discourages any unnecessary steepness in the network function.
In practice, this natural regularization of the global Lipschitz constant may serve to make the steepness of the network function more uniform, such that the global Lipschitz constant will be similar to the local Lipschitz constant.

We show that this is possible in theory, in that for any network for which local robustness can be verified on some set of points using the local Lipschitz constant, there exists a model on which the same points can be certified using the global Lipschitz constant (Theorem~\ref{thm:global_vs_local}).
This suggests that if training is successful, our approach has the same potential using a global bound as using a local bound.

\begin{theorem}\label{thm:global_vs_local}
Let $f$ be a binary classifier that predicts $1 \Longleftrightarrow f(x) > 0$.
Let $K_L(x, \epsilon)$ be the local Lipschitz constant of $f$ at point $x$ with radius $\epsilon$.

Suppose that for some finite set of points, $S$, $\forall x \in S$, $|f(x)| > \epsilon K_L(x, \epsilon)$, i.e., all points in $S$ can be verified via the local Lipschitz constant.

Then there exists a classifier, $g$, with global Lipschitz constant $K_G$, such that $\forall x\in S$,
(1) 
$g$ makes the same predictions as $f$ on $S$, and (2) $|g(x)| > \epsilon K_G$, i.e., all points in $S$ can be verified via the global Lipschitz constant.
\end{theorem}

\begin{figure}[t]
\centering
\includegraphics[page=1, height=0.5\columnwidth]{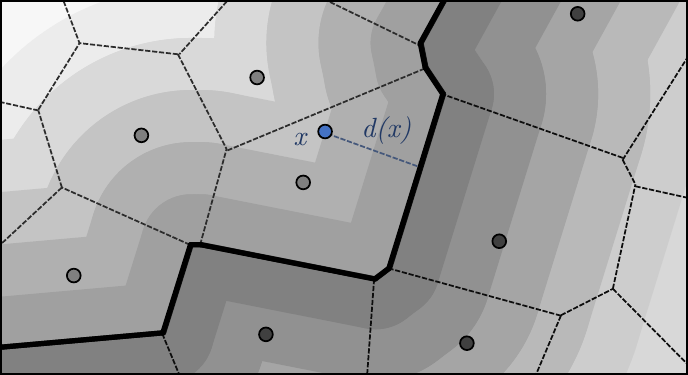}
\caption{Illustration of a function, $g$, constructed to satisfy Theorem~\ref{thm:global_vs_local}. The points in $S$ are shown in light and dark gray, with different shades indicating different labels. The Voronoi tessellation is outlined in black, and the faces belonging to the decision boundary are highlighted in bold. The level curves of $g$ are shown in various shades of gray and correspond to points, $x$, at some fixed distance, $d(x)$, from the decision boundary.}
\label{fig:proof_sketch}
\end{figure}

Theorem~\ref{thm:global_vs_local} is stated for binary classifiers, though the result holds for categorical classifiers as well.
Details on the categorical case and the proof of Theorem~\ref{thm:global_vs_local} can be found in Appendix~\ref{proof:global_vs_local}; however we provide the intuition behind the construction here.
The proof relies on the following lemma, which states that among locally-robust points, points that are classified differently from one another are $2\epsilon$-separated.
The proof of Lemma~\ref{lemma:separation} can be found in Appendix~\ref{proof:separation}.

\begin{lemma}\label{lemma:separation}
Suppose that for some classifier, $F$, and some set of points, $S$, $\forall x \in S$, $F$ is $\epsilon$-locally-robust at $x$.
Then $\forall x_1, x_2\in S$ such that $F(x_1) \neq F(x_2)$, $||x_1 - x_2|| > 2\epsilon$.
\end{lemma}

In the proof of Theorem~\ref{thm:global_vs_local}, we construct a function, $g$, whose output on point $x$ increases linearly with $x$'s minimum distance to any face in the Voronoi tessellation of $S$ that separates points in $S$ with different labels.
An illustration with an example of $g$ is shown in Figure~\ref{fig:proof_sketch}.
Notably, the local Lipschitz constant of $g$ is everywhere the same as the global constant.

However, we note that while Theorem~\ref{thm:global_vs_local} suggests that networks exist \emph{in principle} on which it is possible to use the global Lipschitz constant to certify $2\epsilon$-separated data, it may be that such networks are not easily obtainable via training.
Furthermore, as raised by \citet{huster2018limitations}, an additional potential difficulty in using the global bound for certification is the estimation of the global bound itself.
Methods used for determining an upper bound on the global Lipschitz constant, such as the method presented later in Section~\ref{sec:bounds}, may provide a loose upper bound that is insufficient for verification even when the true bound would suffice.
Nevertheless, our evaluation in Section~\ref{sec:eval} shows that in practice the global bound \emph{can} be used effectively for certification (Section~\ref{sec:eval:accuracy}), and that the bounds obtained on the models trained with our approach are far tighter than those obtained on standard models (Section~\ref{sec:eval:tightness}).

\section{Implementation}\label{sec:implementation}

In this section we describe how GloRo Nets can be trained and implemented.
Section~\ref{sec:training} covers training, as well as the loss functions used in our evaluation, and Section~\ref{sec:bounds} provides detail on how we compute an upper bound on the global Lipschitz constant.

\subsection{Training}\label{sec:training}
Crucially, because certification is intrinsically captured by the GloRo Net's predictions---specifically, the $\bot$ class represents inputs that cannot be verified to be locally robust---a standard learning objective for a GloRo Net corresponds to a robust objective on the original model that was instrumented.
That is, we can train a GloRo Net by simply appending a zero to the one-hot encodings of the original data labels (signifying that $\bot$ is never the correct label), and then optimizing a standard classification learning objective, e.g.,  with cross-entropy loss.
Using this approach, $\margin F$ will be penalized for incorrectly predicting each point, $x$, unless $x$ is both predicted correctly \emph{and} $F$ is $\epsilon$-locally-robust at $x$.

While the above approach is sufficient for training models with competitive VRA, we find that the resulting VRA can be further improved using a loss inspired by TRADES~\cite{zhang19trades}, which balances separately the goals of making \emph{correct} predictions, and making \emph{robust} predictions.
Recent work~\cite{yang20acc_rob_tradeoff} has shown that TRADES effectively controls the local Lipschitz continuity of networks.
While TRADES is implemented using adversarial perturbations, which provide an under-approximation of the robust error, GloRo Nets naturally lend themselves to a variant that uses an over-approximation, as shown in Definition~\ref{def:trades_loss}.

\begin{definition}\label{def:trades_loss}
(TRADES Loss for GloRo Nets)
Given a network, $f$, cross-entropy loss $L_{\text{CE}}$, and parameter, $\lambda$, the GloRo-TRADES loss ($L_T$) of $(x, y)$ is
$$
L_{T}(x,y) = L_{\text{CE}}\big(f(x), y\big) + \lambda D_{\text{KL}}\left(\margin f(x) || f(x)\right)
$$
\end{definition}

Intuitively, $L_T$ combines the normal classification loss with the over-approximate robust loss, assuming that the class predicted by the underlying model is correct.
Empirically we find that using the KL divergence, $D_{\text{KL}}$, in the second term produces the best results, although in many cases using $L_{\text{CE}}$ in both terms works as well.

\subsection{Bounding the Global Lipschitz Constant}\label{sec:bounds}
There has been a great deal of work on calculating upper bounds on the Lipschitz constants of neural networks (see Section~\ref{sec:related} for a discussion).
Our implementation uses the fact that
the product of the spectral norm of each of the individual layers of a feed-forward network provides an upper bound on the Lipschitz constant of the entire network~\cite{SzegedyZSBEGF13}.
That is, if the output at class $i$ of a neural network can be decomposed into a series of $k$ transformations, i.e., $f_i = h^k \circ h^{k-1} \circ \dots \circ h^1$, then Equation~\ref{eq:spectral_norm_bound} holds (where $||\cdot||$ is the spectral norm).
\begin{equation}\label{eq:spectral_norm_bound}
    K_i \leq \prod_{j=1}^{k}{||h^j||}
\end{equation}
In the case of a CNN consisting of convolutional layers, dense layers, and ReLU activations, we use $1$ for the spectral norm of each of the ReLU layers, and we use the power method~\cite{farnia2018generalizable,Gouk2021} to compute the spectral norm of the convolutional and dense layers.
\citeauthor{Gouk2021} also give a procedure for bounding the spectral norm of skip connections and batch normalization layers, enabling this approach on ResNet architectures.
For more complicated networks, there is a growing body of work on computing layer-wise Lipschitz bounds for various types of layers that are commonly used in neural networks~\cite{zou2019lipschitz,fazlyab2019efficient,sedghi2018the,singla19arxiv,miyato2018spectral}.

The power method may need several iterations to converge; however, we can reduce the number of iterations required at each training step by persisting the state of the power method iterates across steps.
While this optimization may not guarantee an upper bound, this fact is inconsequential so long as we still obtain a model that can be certified with a true upper bound that is computed after training; this is actually not unreasonable to expect, presuming the underlying model parameters do not change too quickly.
With a small number of iterations, the additional memory required to compute the Lipschitz constant via this method is approximately the same as to run the network on a single instance.

At test time, the power method must be run to convergence; however, after training, the global Lipschitz bound will remain unchanged and therefore it can be computed 
once in advance. This means that new points can be certified with \emph{no additional non-trivial overhead}.

\paragraph{$\ell_\infty$ Bounds.}
While in this work, we focus on the $\ell_2$ norm, the ideas presented in Section~\ref{sec:method} can be applied to other norms, including the $\ell_\infty$ norm.
However, we find that the analogue of the approximation of the global Lipschitz bound given by Equation~\ref{eq:spectral_norm_bound} is loose in $\ell_\infty$ space.
Meanwhile, a large volume of prior work applies $\ell_\infty$-specific certification strategies that proven effective for $\ell_\infty$ certification~\cite{zhang20crown_ibp,balunovic20colt,Gowal_2019_ICCV}.


\section{Evaluation}\label{sec:eval}
In this section, we present an empirical evaluation of our method.
We first compare GloRo Nets with several certified training methods from the recent literature in Section~\ref{sec:eval:accuracy}.
We also report the training cost, in terms of per-epoch time and peak memory usage, required to train and certify the robustness of our method compared with other competitive approaches (Section~\ref{sec:eval:cost}). 
We end by demonstrating the relative tightness of the estimated Lipschitz bounds for GloRo Nets in Section~\ref{sec:eval:tightness}.
We compare against the KW~\cite{NEURIPS2018_358f9e7b} and BCP~\cite{lee20local_margin} certified training algorithms, which prior work~\cite{lee20local_margin,croce19mmr} reported to achieve the best verified accuracy on MNIST~\cite{lecun2010mnist}, CIFAR-10~\cite{Krizhevsky09learningmultiple} and Tiny-Imagenet~\cite{Le2015TinyIV} relative to other previous certified training methods for $\ell_2$ robustness. 

We train GloRo nets 
to certify robustness against $\ell_2$ perturbations within an $\epsilon$-neighborhood of $0.3$ and $1.58$ for MNIST and $36/255$ for CIFAR-10 and Tiny-Imagenet (these are the $\ell_2$ norm bounds that have been commonly used in the previous literature). 
For each model, we report the \emph{clean accuracy}, i.e., the accuracy without verification on non-adversarial inputs, the \emph{PGD accuracy}, i.e., the accuracy under adversarial perturbations found via the PGD attack~\cite{madry2018towards}, and the \emph{verified-robust accuracy} (VRA), i.e., the fraction of points that are both correctly classified \emph{and} certified as robust. 
For KW and BCP, we report the corresponding best VRAs from the original respective papers when possible, but measure training and certification costs on our hardware for an equal comparison.
We run the PGD attacks using ART~\cite{nicolae2019adversarial} on our models and on any of the models from the prior work for which PGD accuracy is not reported. 
When training BCP models for MNIST with $\epsilon=0.3$, we found a different set of hyperparameters that outperforms those given by \citeauthor{lee20local_margin}.
For GloRo Nets, we found that MinMax activations~\cite{anil19a} performed better than ReLU activations (see Appendix~\ref{appendix:minmax_vs_relu} for more details); for all other models, ReLU activations were used.

Further details on the precise hyperparameters used for training and attacks, the process for obtaining these parameters, and the network architectures are provided in Appendix~\ref{appendix:hyperparams}.
An implementation of our approach is available on GitHub\footnote{Code available at ~{\fontsize{7.5}{11}\selectfont \url{https://github.com/klasleino/gloro}}}.

\begin{figure*}[t]
\centering
\begin{subfigure}[t]{0.62\textwidth}
\resizebox{\textwidth}{!}{%
\begin{tabular}{l|c|ccc|ccc}
\toprule
\textit{method} & Model & Clean (\%) & PGD (\%)& VRA(\%)  & Sec./epoch & \# Epochs & Mem. (MB) \\
\midrule
\multicolumn{8}{c}{\textbf{MNIST} $(\epsilon=0.3)$} \\
\midrule
Standard & \smallmnistmodel &
    99.2 & 96.9 & 0.0 &
    0.3 & 100 & 0.6 \\
GloRo & \smallmnistmodel &
    99.0 &97.8 & \textbf{95.7} &
    0.9 & 500 & 0.7 \\
KW & \smallmnistmodel &
    98.9 & 97.8 & 94.0 &
    66.9 & 100 & 20.2 \\
BCP & \smallmnistmodel &
    93.4 & 89.5 & 84.7 &
    44.8 & 300 & 12.6 \\
\midrule
RS$^*$ & \smallmnistmodel &
    98.8 & \na & 97.4 &
    \na & \na & \na \\
\midrule
\multicolumn{8}{c}{\textbf{MNIST} $(\epsilon=1.58)$} \\
\midrule
Standard & \mnistmodel &
    99.0 & 45.4 & 0.0 &
    0.9 & $\text{42}^\dagger$ & 2.2 \\
GloRo & \mnistmodel &
    97.0 & 81.9 & \textbf{62.8} &
    3.7 & 300 & 2.7 \\
KW & \mnistmodel & 88.1 & 67.9 & 44.5 &
    138.1 & 60 & 84.0 \\
BCP & \mnistmodel &
    92.4 & 65.8 & 47.9 &
    43.4 & 60 & 12.6 \\
\midrule
RS$^*$ & \mnistmodel &
    99.0 & \na & 59.1 &
    \na & \na & \na \\
\midrule
\multicolumn{8}{c}{\textbf{CIFAR-10} $(\epsilon=\nicefrac{36}{255})$} \\
\midrule
Standard & \cifarmodel &
    85.7 & 31.9 & 0.0 &
    1.8 & $\text{115}^\dagger$ & 2.5 \\
GloRo & \cifarmodel &
    77.0 & 69.2 & \textbf{58.4} &
    6.9 & 800 & 3.6 \\
KW & \cifarmodel & 60.1
    & 56.2 & 50.9 &
    516.8 & 60 & 100.9  \\
BCP & \cifarmodel & 
    65.7 & 60.8 & 51.3 &
    47.5 & 200 & 12.7 \\
\midrule
RS$^*$ & \cifarmodel &
    74.1 & \na & 64.2 &
    \na & \na & \na \\
\midrule
\multicolumn{8}{c}{\textbf{Tiny-Imagenet} $(\epsilon=\nicefrac{36}{255})$} \\
\midrule
Standard & \tinyimagenetmodel & 35.9
    & 19.4&0.0 &
    10.7 & $\text{58}^\dagger$ & 6.7 \\
GloRo & \tinyimagenetmodel & 35.5
    & 32.3 & \textbf{22.4} &
    40.3 & 800 & 10.4 \\
KW & \na & \na & \na & \na & \na & \na & \na \\
BCP & \tinyimagenetmodel & 28.8 & 26.6 & 20.1 & 798.8 & 102 & 715.2 \\
\midrule
RS$^*$ & \tinyimagenetmodel &
    23.4 & \na & 16.9 &
    \na & \na & \na \\
\bottomrule
\end{tabular}}
\caption{}\label{fig:eval:train-results}
\end{subfigure}
\hspace{0.1em}
\begin{subfigure}[t]{0.35\textwidth}
\vspace{-12em}
\resizebox{\textwidth}{!}{%
\begin{tabular}{l|c|cc}
\toprule
\textit{method} & Model & Time (sec.) & Mem. (MB)\\
\midrule
GloRo & \cifarmodel & 0.4 & 1.8 \\
KW & \cifarmodel & 2,515.6 & 1,437.5 \\
BCP & \cifarmodel & 5.8 & 19.1 \\
\midrule
RS$^*$ & \cifarmodel & 36,845.5 & 19.8 \\
\bottomrule
\end{tabular}
}
\caption{}\label{fig:eval:cert-results}
\vspace{3em}
\resizebox{\textwidth}{!}{%
\begin{tabular}{l|c|cc}
\toprule
\textit{method} \hspace{1em} & global UB & global LB & local LB \\
\midrule
\multicolumn{4}{c}{\textbf{MNIST} $(\epsilon=1.58)$} \\
\midrule
Standard & $5.4\cdot10^4$ & $1.4\cdot10^2$ & $17.1$ \\
GloRo    & $2.3$          & $1.9$          & $0.8$ \\
\midrule
\multicolumn{4}{c}{\textbf{CIFAR-10} $(\epsilon=\nicefrac{36}{255})$} \\
\midrule
Standard & $1.2\cdot10^7$ & $1.1\cdot10^3$ & $96.2$ \\
GloRo    & $15.8$         & $11.0$         & $3.7$ \\
\midrule
\multicolumn{4}{c}{\textbf{Tiny-Imagenet} $(\epsilon=\nicefrac{36}{255})$} \\
\midrule
Standard & $2.2\cdot10^7$ & $3.6\cdot10^2$ & $40.7$ \\
GloRo    & $12.5$         & $5.9$          & $0.8$ \\
\bottomrule
\end{tabular}
}
\caption{}\label{fig:eval:lipschitz-results}
\end{subfigure}
\caption{
\textbf{(\subref{fig:eval:train-results})}
Certifiable training evaluation results on benchmark datasets. Best results are highlighted in bold.
Randomized Smoothing (RS) is marked with a * superscript to indicate that it provides only a \emph{stochastic} robustness guarantee.
Training cost for RS is omitted as it essentially post-processes standard-trained models (see Appendix~\ref{appendix:memory_usage} for more details).
A $\dagger$ superscript on the number of epochs denotes that an early-stop callback was used to determine convergence.
\textbf{(\subref{fig:eval:cert-results})}
Certification timing and memory usage results on CIFAR-10 ($\epsilon=\nicefrac{36}{255}$).
\textbf{(\subref{fig:eval:lipschitz-results})}
Upper and lower bounds on the global and average local Lipschitz constant.
In (\subref{fig:eval:train-results}) and (\subref{fig:eval:cert-results}), peak GPU Memory usage is calculated per-instance by dividing the total measurement by the training or certification batch size.
}
\end{figure*}

\subsection{Verified Accuracy}\label{sec:eval:accuracy}

We first compare the VRA obtained by GloRo Nets to the VRA achieved by prior deterministic approaches.
KW and BCP have been found to achieve the best VRA on the datasets commonly used in the previous literature~\cite{lee20local_margin,croce19mmr}.
In Appendix~\ref{appendix:comprehensive_vra} we provide a more comprehensive comparison to the VRAs that have been reported in prior work.

Figure~\ref{fig:eval:train-results} gives the best VRA achieved by standard training, GloRo Nets, KW, and BCP on several benchmark datasets and architectures.
In accordance with prior work, we include the clean accuracy and the PGD accuracy as well.
Whereas the VRA gives a lower bound on the number of correctly-classified points that are locally robust, the PGD accuracy serves as an upper bound on the same quantity.
We also provide the (probabilistic) VRA achieved via Randomized Smoothing (RS)~\cite{cohen19smoothing} on each of the datasets in our evaluation, as a comparison to stochastic certification.

We find that GloRo Nets consistently outperform the previous state-of-the-art deterministic VRA.
On MNIST, GloRo Nets outperform all previous approaches with both $\ell_2$ bounds commonly used in prior work ($\epsilon = 0.3$ and $\epsilon = 1.58$).
When $\epsilon = 0.3$, the VRA begins to approach the clean accuracy of the standard-trained model; for this bound, GloRo Nets outperform the previous best VRA (achieved by KW) by nearly two percentage points, accounting for roughly $33\%$ of the gap between the VRA of KW and the clean accuracy of the standard model.
For $\epsilon = 1.58$, GloRo Nets improve upon the previous best VRA (achieved by BCP) by approximately $15$ percentage points---in fact, the VRA achieved by GloRo Nets in this setting even slightly exceeds that of Randomized Smoothing, despite the fact that RS provides only a stochastic guarantee.
On CIFAR-10, GloRo Nets exceed the best VRA (achieved by BCP) by approximately $7$ percentage points.
Finally, on Tiny-Imagenet, GloRo Nets outperform BCP by approximately $2$ percentage points, improving the state-of-the-art VRA by roughly $10\%$.
KW was unable to scale to Tiny-Imagenet due to memory pressure. 

The results achieved by GloRo Nets in Figure~\ref{fig:eval:train-results} are achieved using MinMax activations~\cite{anil19a} rather than ReLU activations, as we found MinMax activations provide a substantial performance boost to GloRo Nets.
We note, however, that both KW and BCP tailor their analysis specifically to ReLU activations, meaning that they would require non-trivial modifications to support MinMax activations.
Meanwhile, GloRo Nets easily support generic activation functions, provided the Lipschitz constant of the activation can be bounded (e.g., the Lipschitz constant of a MinMax activation is $1$).
Moreover, even with ReLU activations, GloRo Nets outperform or match the VRAs of KW and BCP; Appendix~\ref{appendix:minmax_vs_relu} provides these results for comparison.

\subsection{Training and Certification Cost}\label{sec:eval:cost}

A key advantage to GloRo Nets over prior approaches is their ability to achieve state-of-the-art VRA (see Section~\ref{sec:eval:accuracy}) using a global Lipschitz bound.
As discussed in Section~\ref{sec:local_bounds}, this confers performance benefits---both at train and test time---over using a local bound (e.g., BCP), or other expensive approaches (e.g., KW).

Figure~\ref{fig:eval:train-results} shows the cost of each approach both in time per epoch and in memory during training (results given for CIFAR-10).
All timings were taken on a machine using a Geforce RTX 3080 accelerator, 64 GB memory, and Intel i9 10850K CPU, with the exception of those for the KW~\cite{NEURIPS2018_358f9e7b} method, which were taken on a Titan RTX card for toolkit compatibility reasons.
Appendix~\ref{appendix:memory_usage} provides further details on how memory usage was measured.
Because different batch sizes were used to train and evaluate each model, we control for this by reporting the memory used \emph{per instance in each batch}.
The cost for standard training is included for comparison.
The training cost of RS is omitted, as RS does not use a specialized training procedure, and is thus comparable to standard training.
Appendix~\ref{appendix:memory_usage} provides more information on this point.

We see that KW is the most expensive approach to train, requiring tens to hundreds of seconds per epoch and roughly $35\times$ more memory per batch instance than standard training.
BCP is less expensive than KW, but still takes nearly one minute per epoch on MNIST and CIFAR and 15 minutes on Tiny-Imagenet, and uses anywhere between $5$-$106\times$ more memory than standard training.

Meanwhile, the cost of GloRo Nets is more comparable to that of standard training than of KW or BCP, taking only a few seconds per epoch, and at most $50\%$ more memory than standard training.
Because of its memory scalability, we were able to use a larger batch size with GloRo Nets.
As a result, more epochs were required during training however, this did not outweigh the significant reduction in time per epoch, as the total time for training was still only at most half of the total time for BCP.

Figure~\ref{fig:eval:cert-results} shows the cost of each approach both in the time required to certify the entire test set and in the memory used to do so (results given for CIFAR-10).
KW is the most expensive deterministic approach in terms of time and memory, followed by BCP.
Here again, GloRo Nets are far superior in terms of cost, making certified predictions over $14\times$ faster than BCP with less than a tenth of the memory, and over 6,000$\times$ faster than KW.
We thus conclude that GloRo Nets are the most scalable state-of-the-art technique for robustness certification.

As reported in Figure~\ref{fig:eval:train-results}, Randomized Smoothing typically outperforms the VRA achieved by GloRo Nets, and is also inexpensive to train; though the VRA achieved by RS reflects a stochastic guarantee rather than a deterministic one.
However, we see in Figure~\ref{fig:eval:cert-results} that GloRo Nets are \emph{several orders of magnitude} faster at certification than RS.
GloRo Nets perform certification in a single forward pass, enabling certification of the entire CIFAR-10 test set in \emph{under half a second};
on the other hand, RS requires tens of thousands of samples to provide confident guarantees, reducing throughput by orders of magnitude and requiring over \emph{ten hours} to certify the same set of instances.

\subsection{Lipschitz Tightness}\label{sec:eval:tightness}

Theorem~\ref{thm:global_vs_local} demonstrates that a global Lipschitz bound is theoretically sufficient for certifying $2\epsilon$-separated data.
However, as discussed in Section~\ref{sec:local_bounds}, there may be several practical limitations making it difficult to realize a network satisfying Theorem~\ref{thm:global_vs_local}; we now assess how these limitations are borne out in practice by examining the Lipschitz bounds that GloRo Nets use for certification.

\citet{weng18fastlip} report that an upper bound on the global Lipschitz constant is not capable of certifying robustness for a non-trivial radius.
While this is true of models produced via \emph{standard training}, GloRo Nets impose a strong implicit regularization on the global Lipschitz constant.
Indeed, Figure~\ref{fig:eval:lipschitz-results} shows that the global upper bound is several orders of magnitude smaller on GloRo Nets than on standard networks.

Another potential limitation of using an upper bound of the global Lipschitz constant is the bound itself~\cite{huster2018limitations}.
Figure~\ref{fig:eval:lipschitz-results} shows that a lower bound of the Global Lipschitz constant, obtained via optimization, reaches an impressive $83\%$ of the upper bound on MNIST, meaning that the upper bound is fairly tight. 
On CIFAR-10 and Tiny-Imagenet the lower bound reaches approximately $70\%$ and $47\%$ of the upper bound, respectively.
However, on a standard model, the lower bound is potentially orders of magnitude looser.
These results show there is still room for improvement; for example, using the lower bound in place of the upper bound would lead to roughly a $10\%$ increase in VRA on CIFAR-10, from $58\%$ to $64\%$.
However, the fact that the bound is tighter for GloRo Nets suggests the objective imposed by the GloRo Net helps by incentivizing parameters for which the upper bound estimate is sufficiently tight for verification.

Finally, we compare the global upper bound to an empirical lower bound of the local Lipschitz constant.
The local lower bound given in Figure~\ref{fig:eval:lipschitz-results} reports the \emph{mean} local Lipschitz constant found via optimization in the $\epsilon$-balls centered at each of the test points.
In the construction given for the proof of Theorem~\ref{thm:global_vs_local}, the local Lipschitz constant is the same as the global bound at all points.
While the results in Figure~\ref{fig:eval:lipschitz-results} show that this may not be entirely achieved in practice, the ratio of the local lower bound to the global upper bound is essentially zero in the standard models, compared to $6$-$35\%$ in the GloRo Nets, establishing that the upper bound is again much tighter for GloRo Nets.
Still, this suggests that a reasonably tight estimate of the local bound may yet help improve the VRA of a GloRo Net at runtime, although this is a challenge in its own right.
Intriguingly, GloRo Nets outperform BCP, which utilizes a \emph{local} Lipschitz bound for certification at train and test time, suggesting that GloRo Nets provide a better objective for certifiable robustness despite using a looser bound during training.

We provide further discussion of the upper and lower bounds, and details for how the lower bounds were obtained in Appendix~\ref{appendix:lower_bounds}.


\vspace{0.80em}
\section{Related Work}\label{sec:related}

Utilizing the Lipschitz constant to certify robustness has been studied in several instances of prior work.
On discovering the existence of adversarial examples, \citet{SzegedyZSBEGF13} analyzed the sensitivity of neural networks using a global Lipschitz bound, explaining models' ``blind spots'' partially in terms of large bounds and suggesting Lipschitz regularization as a potential remedy.
\citet{huster2018limitations} noted the potential limitations of using global bounds computed layer-wise according to Equation~\ref{eq:spectral_norm_bound}, and showed experimentally that direct regularization of the Lipschitz constant by penalizing the weight norms of a two-layer network yields subpar results on MNIST.
While Theorem~\ref{thm:global_vs_local} does not negate their concern, as it may not always be feasible to compute a tight enough bound using Equation~\ref{eq:spectral_norm_bound}, our experimental results show to the contrary that global bounds can suffice to produce models with at least comparable utility to several more expensive and complicated techniques.
More recently, \citet{yang20acc_rob_tradeoff} showed that robustness and accuracy need not be at odds on common benchmarks when locally-Lipschitz functions are used, and call for further investigation of methods that impose this condition while promoting generalization.
Our results show that globally-Lipschitz functions, which bring several practical benefits (Section~\ref{sec:local_bounds}), are a promising direction as well.

Lipschitz constants have been applied previously for fast post-hoc certification~\cite{weng18fastlip,hein2017lipschitz,weng2018evaluating}.
While our work relies on similar techniques, our exclusive use of the global bound means that no additional work is needed at inference time.
Additionally, we apply this certification only to networks that have been optimized for it.

There has also been prior work seeking to use Lipschitz bounds, or close analogues, during training to promote robustness~\cite{tsuzuku18margin,raghunathan2018certified,cisse17a,cohen2019universal,anil19a,pauli21control,qin19locallin,finlay_2019_scaleable,lee20local_margin,Gouk2021,singla19arxiv,farnia2018generalizable}.
\citet{cisse17a} introduced Perseval networks, which enforce contractive Lipschitz bounds on all layers by orthonormalizing their weights.
\citet{anil19a} proposed replacing ReLU activations with sorting activations to construct a class of \emph{universal Lipschitz approximators}, that that can approximate any Lipschitz-bounded function over a given domain, and \citet{cohen2019universal} subsequently studied the application to robust training; these advances in architecture complement our work, as noted in Appendix~\ref{appendix:minmax_vs_relu}.

The closest work in spirit to ours is Lipschitz Margin Training (LMT)~\cite{tsuzuku18margin}, which also uses global Lipschitz bounds to train models that are more certifiably robust.
The approach works by constructing a loss that adds $\sqrt{2}\epsilon K_G$ to all logits other than that corresponding to the ground-truth class.
Note that this is different from GloRo Nets, which add a \emph{new logit} defined by the \emph{predicted} class at $x$.
In addition to providing different gradients than those of LMT's loss, our approach avoids penalizing logits corresponding to boundaries distant from $x$.
In practical terms, \citet{lee20local_margin} showed that LMT yields lower verified accuracy than more recent methods that use local Lipschitz bounds~\cite{lee20local_margin} or dual networks~\cite{NEURIPS2018_358f9e7b}, while Section~\ref{sec:eval:accuracy} shows that our approach can provide greater verified accuracy than either.
LMT's use of global bounds means its cost is comparable to our approach.

More recently, \citet{lee20local_margin} explored the possibility of training networks against local Lipschitz bounds, motivated by the fact that the global bound may vastly exceed a typical local bound on some networks.
They showed that a localized refinement of the global spectral norm of the network offers a reasonable trade-off of precision for cost, and were able to achieve competitive, and in some cases superior, verified accuracy to prior work.
Theorem~\ref{thm:global_vs_local} shows that in principle, the difference in magnitude between local and global bounds may not matter for robust classification.
Moreover, while it is true that the bounds computed by Equation~\ref{eq:spectral_norm_bound} may be loose on some models, our experimental results suggest that it is possible in many cases to mitigate this limitation by training against a global bound with the appropriate loss.
The advantages of doing so are apparent in the cost of both training and certification, where the additional overhead involved with computing tighter local bounds is an impediment to scalability.

Finally, several other methods have been proposed for training $\ell_2$-certifiable networks that are not based on Lipschitz constants.
For example, \citet{wong2017provable} use an LP-based approach that can be optimized relatively efficiently using a \emph{dual network}, \citet{croce19mmr} and \citet{madry2018towards} propose training routines based on maximizing the size of the linear regions within a network, and \citet{mirman18diffai} propose a method based on abstract interpretation.

\paragraph{Randomized Smoothing.}
The certification methods discussed thus far provide \emph{deterministic} robustness guarantees.
By contrast, another recent approach, Randomized Smoothing~\cite{cohen19smoothing,lecuyer18smoothing}, provides \emph{stochastic} guarantees---that is, points are certified as \emph{robust with high probability} (i.e., the probability can be bounded from below).
Randomized Smoothing has been found to achieve better VRA performance than any deterministic certification method, including GloRo Nets.
However, GloRo Nets compare favorably to Randomized Smoothing in a few key ways.

First, the fact that GloRo Nets provide a deterministic guarantee is an advantage in and of itself.
In safety-critical applications, it may not be considered acceptable for a small fraction of adversarial examples to go undetected;
meanwhile, Randomized Smoothing is typically evaluated with a false positive rate around $0.1\%$~\cite{cohen19smoothing}, meaning that instances of incorrectly-certified points are to be expected in validation sets with thousands of points.

Furthermore, as demonstrated in Section~\ref{sec:eval:cost}, GloRo Nets have far superior run-time cost.
Because Randomized Smoothing does not explicitly represent the function behind its robust predictions, points must be evaluated and certified using as many as 100,000 samples~\cite{cohen19smoothing}, reducing throughput by several orders of magnitude.
Meanwhile, GloRo Nets can certify a batch of points in \emph{a single forward pass}.


\section{Conclusion}\label{sec:conclusion}

In this work, we provide a method for training certifiably-robust neural networks that is simple, fast, memory-efficient, and that yields state-of-the-art deterministic verified accuracy.
Our approach is particularly efficient because of its effective use of global Lipschitz bounds, and while we prove that the potential of our approach is in theory not limited by the global Lipschitz constant itself, it remains an open question as to whether our bounds on the Lipschitz constant can be tightened, or if additional training techniques can help unlock its remaining potential.
Finally, we note that if instances arise where a global bound is not sufficient in practice, costlier post-hoc certification techniques may be complimentary, as a fall-back.

\section*{Acknowledgments}
The work described in this paper has been supported by the Software Engineering Institute under its FFRDC Contract No. FA8702-15-D-0002 with the U.S. Department of Defense, as well as DARPA and the Air Force Research Laboratory under agreement number FA8750-15-2-0277.

\bibliography{grib}
\bibliographystyle{icml2020}

\clearpage

\appendix


\setcounter{figure}{0}
\setcounter{table}{0}
\setcounter{equation}{2}
\setcounter{definition}{3} 
\setcounter{theorem}{3} 
\numberwithin{figure}{section}
\numberwithin{table}{section}

\twocolumn[
    \icmltitle{
        Appendix
    }



    \icmlsetsymbol{equal}{*}
]

\section{Proofs}

\subsection{Proof of Theorem
\ref{thm:local_robustness}
}\label{proof:local_robustness}

\paragraph{Theorem~\ref{thm:local_robustness}.}
\textit{%
If $\margin F(x) \neq \bot$, then $\margin F(x) = F(x)$ and $F$ is $\epsilon$-locally-robust at $x$.
}


\begin{proof}
Let $j = F(x)$.
Assume that $\margin F(x) \neq \bot$; 
this happens only if one of the outputs of $f$ is greater than $\margin f_\bot(x)$ --- from the definition of $f_\bot(x)$, it is clear that only $f_j(x)$ can be greater than $\margin f_\bot(x)$.
Therefore $f_j(x) > \margin f_\bot(x)$, and so $\margin F(x) = j = F(x)$.

Now assume $x'$ satisfies $||x - x'|| \leq \epsilon$.
Let $K_i$ be an upper bound on the Lipschitz constant of $f_i$.
Then, $\forall i$
\begin{align}
    & \frac{|f_i(x) - f_i(x')|}{\epsilon} \leq 
    \frac{|f_i(x) - f_i(x')|}{||x - x'||} \leq
    K_i \nonumber \\
    \Longrightarrow~~& |f_i(x) - f_i(x')| \leq K_i \epsilon 
    \label{step:apply_lipschitz}
\end{align}
We proceed to show that for any such $x'$, $F(x')$ is also $j$.
In other words, $\forall i \neq j$, $f_i(x') < f_j(x')$.
By applying the definition of the Lipschitz constant as in (\ref{step:apply_lipschitz}), we obtain (\ref{step:i_to_K_eps}). 
Next, (\ref{step:to_f_bot}) follows from the fact that $\margin f_\bot(x) = \max_{i\neq j}\left\{y_i + (K_i + K_j)\epsilon\right\}$.
We then obtain (\ref{step:apply_prediction}) from the fact that $f_j(x) > \margin f_\bot(x)$, as observed above.
Finally, we again apply (\ref{step:apply_lipschitz}) to obtain (\ref{step:K_eps_to_j}).
\begin{align}
    f_i(x') 
    & \leq f_i(x) + |f_i(x) - f_i(x')| \leq f_i(x) + K_i \epsilon 
    \label{step:i_to_K_eps} \\
    & \leq \margin f_\bot(x) - K_j \epsilon 
    \label{step:to_f_bot} \\
    & < f_j(x) - K_j \epsilon 
    \label{step:apply_prediction} \\
    & \leq f_j(x) - |f_j(x) - f_j(x')| \leq f_j(x')
    \label{step:K_eps_to_j}
\end{align}

Therefore, $f_i(x') < f_j(x')$, and so $F(x') = j$.
This means that $F$ is locally robust at $x$.
\end{proof}

\subsection{Tighter Bounds for Theorem
\ref{thm:local_robustness}
}\label{proof:tighter_bound}

Note that in the formulation of GloRo Nets given in Section~\ref{sec:gloro_nets}, we assume that the predicted class, $j$, will decrease by the maximum amount within the $\epsilon$-ball, while all other classes increase by their respective maximum amounts. 
This is a conservative assumption that guarantees local robustness;
however, in practice, we can dispose of this assumption by instead calculating the Lipschitz constant of the margin by which the logit of the predicted class surpasses the other logits, $f_j - f_i$.

The \emph{margin Lipschitz constant} of $f$, defined for a pair of classes, $i\neq j$, is given by Definition~\ref{def:margin_lip}.
\begin{definition}{Margin Lipschitz Constant}\label{def:margin_lip}
For network, $f : \mathds{R}^n \rightarrow \mathds{R}^m$, and classes $i\neq j \in [m]$, $K^*_{ij}$ is an upper bound on the margin Lipschitz constant of $f$ if $\forall x_1, x_2$,
$$
\frac{|f_j(x_1) - f_i(x_1) - (f_j(x_2) - f_i(x_2))|}{||x_1 - x_2||} \leq K^*_{ij}
$$
\end{definition}

We now define a variant of GloRo Nets (Section~\ref{sec:gloro_nets}) as follows:
For input, $x$, let $j = F(x)$, i.e., $j$ is the label assigned by the underlying model to be instrumented.
Define $\margin f_i(x) ::= f_i(x)$, and $\margin f_\bot(x) ::= \max_{i\neq j}\{f_i(x) + \epsilon K^*_{ij}\}$.

\begin{theorem}\label{thm:gloro_variant}
Under this variant, if $\margin F(x) \neq \bot$, then $\margin F(x) = F(x)$ and $F$ is $\epsilon$-locally-robust at $x$.
\end{theorem}

\begin{proof}
The proof is similar to the proof of Theorem~\ref{thm:local_robustness} (Appendix~\ref{proof:local_robustness}).
Let $j = F(x)$.
As before, when $\margin F(x) \neq \bot$, we see that $\margin F(x) = j = F(x)$.

Now assume $x'$ satisfies $||x - x'|| \leq \epsilon$.
Let $K^*_{ij}$ be an upper bound on the margin Lipschitz constant.
Then, $\forall i$
\begin{equation}\label{step:apply_m_lipschitz}
    |f_j(x) - f_i(x) - (f_j(x') - f_i(x'))| \leq K^*_{ij} \epsilon
\end{equation}

We proceed to show that for any such $x'$, $F(x')$ is also $j$.
In other words, $\forall i \neq j$, $f_i(x') < f_j(x')$.
By applying (\ref{step:apply_m_lipschitz}), we obtain (\ref{step:to_Kij_eps}). 
Next, (\ref{step:to_f_bot_m}) follows from the fact that $\margin f_\bot(x) = \max_{i\neq j}\left\{f_i(x) + K^*_{ij}\epsilon\right\}$.
We then obtain (\ref{step:apply_prediction_m}) from the fact that $f_j(x) > \margin f_\bot(x)$, as $\margin F(x) = j \neq \bot$.
\begin{align}
         &\cancel{f_i(x)} + f_j(x) - \cancel{f_i(x)} - f_j(x') + f_i(x') \nonumber \\
    \leq~& f_i(x) + |f_j(x) - f_i(x) - (f_j(x') - f_i(x'))| \nonumber \\
    \leq~& f_i(x) + K^*_{ij} \epsilon \label{step:to_Kij_eps}\\
    \leq~& \margin f_\bot(x) \label{step:to_f_bot_m} \\
       <~& f_j(x) \label{step:apply_prediction_m}
\end{align}
Rearranging terms, we obtain that $f_i(x') < f_j(x')$. Thus, $F(x') = j$; this means that $F$ is locally robust at $x$.
\end{proof}

\subsection{Proof of Theorem
\ref{thm:global_robustness}
}\label{proof:global_robustness}

\paragraph{Theorem~\ref{thm:global_robustness}.}
\textit{%
$\margin[\nicefrac{\epsilon}{2}]{F}$ is $\epsilon$-globally-robust.
}


\begin{proof}
Assume $x_1$ and $x_2$ satisfy $||x_1 - x_2|| \leq \epsilon$.
Let $\margin[\nicefrac{\epsilon}{2}]{F}(x_1) = c_1$ and $\margin[\nicefrac{\epsilon}{2}]{F}(x_2) = c_2$.

If $c_1 = \bot$ or $c_2 = \bot$, global robustness is trivially satisfied.

Consider the case where $c_1 \neq \bot$, $c_2 \neq \bot$.
Let $x'$ be the midpoint between $x_1$ and $x_2$, i.e., $x' = (x_1 + x_2) / 2$.
Thus
\begin{align*}
    & ||x_1 - x'|| = \left|\left|\frac{x_1 - x_2}{2}\right|\right| = \frac{||x_1 - x_2||}{2} \leq \frac{\epsilon}{2}.
\end{align*}
By Theorem~\ref{thm:local_robustness}, this implies $F(x') = c_1$.
By the same reasoning, $||x_2 - x'|| \leq \nicefrac{\epsilon}{2}$, implying that $F(x') = c_2$. Thus, $c_1 = c_2$, so global robustness holds.
\end{proof}

\subsection{Proof of Theorem
\ref{thm:global_vs_local}
}\label{proof:global_vs_local}

\paragraph{Theorem~\ref{thm:global_vs_local}.}
\textit{%
Let $f$ be a binary classifier that predicts $1 \Longleftrightarrow f(x) > 0$.
Let $K_L(x, \epsilon)$ be the local Lipschitz constant of $f$ at point $x$ with radius $\epsilon$.
}

\textit{%
Suppose that for some finite set of points, $S$, $\forall x \in S$, $|f(x)| > \epsilon K_L(x, \epsilon)$, i.e., all points in $S$ can be verified via the local Lipschitz constant.
}

\textit{%
Then there exists a classifier, $g$, with global Lipschitz constant $K_G$, such that $\forall x\in S$,
(1) 
$g$ makes the same predictions as $f$ on $S$, and (2) $|g(x)| > \epsilon K_G$, i.e., all points in $S$ can be verified via the global Lipschitz constant.
}

\begin{figure}[t]
\centering
\includegraphics[page=1, height=0.5\columnwidth]{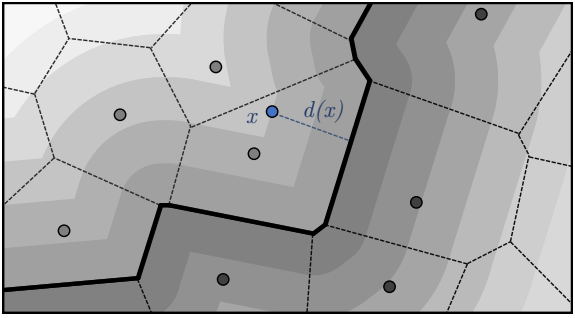}
\caption{Illustration of a function, $g$, constructed to satisfy Theorem~\ref{thm:global_vs_local}. The points in $S$ are shown in light and dark gray, with different shades indicating different labels. The Voronoi tessellation is outlined in black, and the faces belonging to the decision boundary are highlighted in bold. The level curves of $g$ are shown in various shades of gray and correspond to points, $x$, at some fixed distance, $d(x)$, from the decision boundary.}
\label{fig:proof_sketch}
\end{figure}


\begin{proof}
Let $T$ be the Voronoi tessellation generated by the points in $S$.
Each Voronoi cell, $C_j\in T$, corresponds to the set of points that are closer to $p_j\in S$ than to any other point in $S$; and the face, $F_{ij}\in T$, which separates cells $C_i$ and $C_j$, corresponds to the set of points that are equidistant from $p_i$ and $p_j$.

Let $B = \left\{F_{ij} : \text{sign}(f(p_i)) \neq \text{sign}(f(p_j))\right\}$, i.e., the set of faces in the Voronoi tessellation that separate points that are classified differently by $f$ (note that $B$ corresponds to the boundary of the 1-nearest-neighbor classifier for the points in $S$).

Consider a point, $x$.
Let $p_x\in S$ be the closest point in $S$ to $x$, i.e., the point corresponding to the Voronoi cell containing $x$.
Let $d(x) = ||\proj{x}{B} - x||$; that is, $d(x)$ is the minimum distance from $x$ to any point in any of the faces in $B$.
Then define
$$
g(x) = \text{sign}\big(f(p_x)\big)\frac{d(x)}{\epsilon}
$$
First, observe that $g(x) > 0 \Longleftrightarrow f(x) > 0$ follows from the fact that $d(x)$ and $\epsilon$ are non-negative, thus the sign of $g(x)$ is derived from the sign of $f(x)$.

Next, we show that the global Lipschitz constant of $g$, $K_G$, is at most $\nicefrac{1}{\epsilon}$, that is, $\forall x_1, x_2$,
$$
\frac{|g(x_1) - g(x_2)|}{||x_1 - x_2||} \leq \frac{1}{\epsilon}
$$
Consider two points, $x_1$ and $x_2$, and let $p_1$ and $p_2$ be the points in $S$ corresponding to the respective Voronoi cells of $x_1$ and $x_2$.

First, consider the case where $\text{sign}(f(p_1)) \neq \text{sign}(f(p_2))$, i.e., $x_1$ and $x_2$ are on opposite sides of the boundary, $B$.
In this case $|g(x_1) - g(x_2)| = \nicefrac{(d(x_1) + d(x_2))}{\epsilon}$, and thus it suffices to show that $d(x_1) + d(x_2) \leq ||x_1 - x_2||$.

Assume for the sake of contradiction that $d(x_1) + d(x_2) > ||x_1 - x_2||$.
Note that because $x_1$ and $x_2$ belong in Voronoi cells with different classifications from $f$, the line segment connecting $x_1$ and $x_2$ must cross the boundary, $B$, at some point $c$.
Therefore, $||x_1 - c|| + ||x_2 - c|| = ||x_1 - x_2|| < d(x_1) + d(x_2)$; without loss of generality, this implies that $||x_1 - c|| < d(x_1)$.
But since $c \in F \in B$, this contradicts that $d(x_1)$ is the minimum distance from $x_1$ to $B$.\lightning

Next, consider the case where $\text{sign}(f(p_1)) = \text{sign}(f(p_2))$.
In this case $|g(x_1) - g(x_2)| = \nicefrac{|d(x_1) - d(x_2)|}{\epsilon}$, and thus, without loss of generality, it suffices to show that $d(x_1) - d(x_2) \leq ||x_1 - x_2||$.

Assume for the sake of contradiction that $d(x_1) - d(x_2) > ||x_1 - x_2||$.
Thus $d(x_1) > ||x_1 - x_2|| + d(x_2)$.
However, this suggests that we can take the path from $x_1$ to $x_2$ to $B$ with a smaller total distance than $d(x_1)$, contradicting that $d(x_1)$ is the minimum distance from $x_1$ to $B$.\lightning

We now show that $\forall p\in S$, $|g(p)| \geq 1$, i.e., $d(p) \geq \epsilon$.
In other words, we must show that the distance of any point, $p\in S$ to the boundary, $B$, is at least $\epsilon$.
Consider a point, $x$, on some face, $F_{ij}\in B$.
This point is equidistant from $p_i$ and $p_j\in S$, on which $f$ makes different predictions; and every other point in $S$ is at least as far from $x$ as $p_i$ and $p_j$.
I.e., $||p_i - x|| = ||p_j - x|| \leq ||p - x||$, $\forall p\in S$.
By the triangle inequality, $2||p_i  - x|| \geq ||p_i - p_j||$, and
by Lemma~\ref{lemma:separation}, $||p_i - p_j|| \geq 2\epsilon$.
Thus $||p - x|| \geq \epsilon$, $\forall p \in S$; therefore every point on the boundary is at least $\epsilon$ from $p\in S$.

Putting everything together, we have that $\forall p\in S$, $|g(p)| \geq 1 \geq \epsilon K_G$.
\end{proof}

Note that while Theorem~\ref{thm:global_vs_local} is stated for binary classifiers, the result holds for categorical classifiers as well.
We can modify the construction of $g$ from the above proof in a straightforward way to accommodate categorical classifiers.
In the case where there are $m$ different classes, the output of $g$ has $m$ dimensions, each corresponding to a different class.
Then, for $x$ in a Voronoi cell corresponding to $p_x\in S$ with label, $j$, we define $g_j(x) ::= \nicefrac{d(x)}{\epsilon}$ and $g_i(x) ::= 0$ $\forall i \neq j$.
We can see that, for all pairs of classes, $i$ and $j$, the Lipschitz constant of $g_i - g_j$ in this construction is the same as the Lipschitz constant of $g$ in the above proof, since only one dimension of the output of $g$ changes at once.
Thus, we can use the global bound suggested in Appendix~\ref{proof:tighter_bound} to certify the points in $S$.

\vspace{1.5em}
\subsection{Proof of Lemma
\ref{lemma:separation}
}\label{proof:separation}

\paragraph{Lemma~\ref{lemma:separation}.}
\textit{%
Suppose that for some classifier, $F$, and some set of points, $S$, $\forall x \in S$, $F$ is $\epsilon$-locally-robust at $x$.
Then $\forall x_1, x_2\in S$ such that $F(x_1) \neq F(x_2)$, $||x_1 - x_2|| > 2\epsilon$.
}


\begin{proof}
Suppose that for some classifier, $F$, and some set of points, $S$, $\forall x \in S$, $F$ is $\epsilon$-locally-robust at $x$.
Assume for the sake of contradiction that $\exists x_1, x_2 \in S$ such that $F(x_1) \neq F(x_2)$ but $||x_1 - x_2|| \leq 2\epsilon$.
Consider the midpoint between $x_1$ and $x_2$, $x' = (x_1 + x_2) / 2$.
Note that 
$$||x' - x_1|| = \frac{||x_1 - x_2||}{2} \leq \epsilon$$
Therefore, since $F$ is $\epsilon$-locally-robust at $x_1$, $F(x') = F(x_1)$.
By the same argument, $F(x') = F(x_2)$.
But this contradicts that $F(x_1) \neq F(x_2)$.\lightning
\end{proof}

\vspace{1.5em}
\section{Hyperparameters}\label{appendix:hyperparams}

\begin{table*}[!t]
    \begin{center}
    \resizebox*{!}{0.475\dimexpr\textheight-2\baselineskip\relax}{%
    \small
    \begin{tabular}{ccccccccc}
    \toprule
    \textit{architecture}  & \textit{dataset}& \multicolumn{2}{c}{\textit{data augmentation}} &\textit{warm-up}& \textit{batch size} & \# \textit{epochs} & $\epsilon_{\text{train}}$ & $\epsilon_{\text{test}}$\\
    \midrule
    \multirow{5}{*}{\shortstack[l]{\mnistmodelsmall\\ \\ GloRo}}  &MNIST& \multicolumn{2}{c}{\texttt{\texttt{None}}} & 0 & 512 & 500 & 0.3 & 0.3\\
     \cmidrule{2-9}\\
     & \textit{initialization} & \textit{init\_lr} & \multicolumn{2}{c}{\textit{lr\_decay}} & \textit{loss}  & $\epsilon$ \textit{schedule}& \textit{power\_iter}\\
     \cmidrule{2-9}\\
     &\texttt{\texttt{orthogonal}} & 1e-3 & \multicolumn{2}{c}{\texttt{decay\_to\_1e-6}} & \texttt{0.1,2.0,500}& \texttt{single}&5\\

    \midrule
    \midrule

    \textit{architecture}  & \textit{dataset}& \multicolumn{2}{c}{\textit{data augmentation}} &\textit{warm-up}& \textit{batch size} & \# \textit{epochs} & $\epsilon_{\text{train}}$ & $\epsilon_{\text{test}}$\\
    \midrule
    \multirow{5}{*}{\shortstack[l]{\mnistmodel\\ \\ GloRo}}  &MNIST& \multicolumn{2}{c}{\texttt{\texttt{None}}} & 0 & 512 & 200 & 1.74 & 1.58\\
     \cmidrule{2-9}\\
     & \textit{initialization} & \textit{init\_lr} & \multicolumn{2}{c}{\textit{lr\_decay}} & \textit{loss}  & $\epsilon$ \textit{schedule}& \textit{power\_iter}\\
     \cmidrule{2-9}\\
     &\texttt{\texttt{orthogonal}} & 1e-3 & \multicolumn{2}{c}{\texttt{decay\_to\_1e-6}} & \texttt{1.5}& \texttt{log}&5\\

    \midrule
    \midrule

    \textit{architecture} & \textit{dataset} & \multicolumn{2}{c}{\textit{data augmentation}} &\textit{warm-up}& \textit{batch size} & \# \textit{epochs} & $\epsilon_{\text{train}}$ & $\epsilon_{\text{test}}$\\
    \midrule
    \multirow{5}{*}{\shortstack[l]{\cifarmodel\\ \\ GloRo}} &CIFAR-10 & \multicolumn{2}{c}{\texttt{tfds}} & 0 & 512 &800&0.141 &0.141 \\
     \cmidrule{2-9}\\
     & \textit{initialization} & \textit{init\_lr} & \multicolumn{2}{c}{\textit{lr\_decay}} & \textit{loss}  & $\epsilon$ \textit{schedule}& \textit{power\_iter}\\
     \cmidrule{2-9}\\
     &\texttt{orthogonal}& 1e-3 & \multicolumn{2}{c}{\texttt{decay\_to\_1e-6}}&\texttt{1.2} &\texttt{log} &5\\

     \midrule
     \midrule

     \textit{architecture} & \textit{dataset} & \multicolumn{2}{c}{\textit{data augmentation}} &\textit{warm-up}& \textit{batch size} & \# \textit{epochs} & $\epsilon_{\text{train}}$ & $\epsilon_{\text{test}}$ \\
     \midrule
     \multirow{5}{*}{\shortstack[l]{\tinyimagenetmodel\\ \\ GloRo}} &Tiny-Imagenet &\multicolumn{2}{c}{\texttt{default}} & 0 & 512& 800& 0.141& 0.141\\
     \cmidrule{2-9}
      & \textit{optimizer} & \textit{init\_lr} & \multicolumn{2}{c}{\textit{lr\_decay}} & \textit{loss}  & $\epsilon$ \textit{schedule}& \textit{power\_iter}\\
      \cmidrule{2-9}\\
      & \texttt{default} & 1e-4 &\multicolumn{2}{c}{\texttt{decay\_to\_5e-6}}  &\texttt{1.2,10,800} &\texttt{log} & 5\\

    \bottomrule
    \end{tabular}}
    \end{center}
    \caption{
    Hyperparameters used for training (MinMax) GloRo Nets.}\label{appendix:fig:supply:hyperparameters-minmax}
\end{table*}

\begin{table*}[!t]
    \begin{center}
    \resizebox*{!}{0.95\dimexpr\textheight-2\baselineskip\relax}{%
    \small
    \begin{tabular}{ccccccccc}
    \toprule
    \textit{architecture}  & \textit{dataset}& \multicolumn{2}{c}{\textit{data augmentation}} &\textit{warm-up}& \textit{batch size} & \# \textit{epochs} & $\epsilon_{\text{train}}$ & $\epsilon_{\text{test}}$\\
    \midrule
    \multirow{5}{*}{\shortstack[l]{\mnistmodelsmall\\ \\ GloRo (CE)}}  &MNIST& \multicolumn{2}{c}{\texttt{None}} & 0 & 256& 500& 0.45& 0.3\\
     \cmidrule{2-9}\\
     & \textit{initialization} & \textit{init\_lr} & \multicolumn{2}{c}{\textit{lr\_decay}} & \textit{loss}  & $\epsilon$ \textit{schedule}& \textit{power\_iter}\\
     \cmidrule{2-9}\\
     &\texttt{default} & 1e-3 & \multicolumn{2}{c}{\texttt{decay\_to\_1e-6}} & \texttt{CE}& \texttt{single}&10\\

    \midrule
    \midrule

    \textit{architecture} & \textit{dataset}&  \multicolumn{2}{c}{\textit{data augmentation}} &\textit{warm-up}& \textit{batch size} & \# \textit{epochs} & $\epsilon_{\text{train}}$ & $\epsilon_{\text{test}}$ \\
    \midrule
    \multirow{5}{*}{\shortstack[l]{\mnistmodelsmall\\ \\ GloRo (T)}}  &MNIST& \multicolumn{2}{c}{\texttt{None}} & 0 & 256& 500& 0.45& 0.3\\
     \cmidrule{2-9}\\
     & \textit{initialization}& \textit{init\_lr} & \multicolumn{2}{c}{\textit{lr\_decay}} & \textit{loss}  & $\epsilon$ \textit{schedule}& \textit{power\_iter}\\
     \cmidrule{2-9}\\
     &\texttt{default} & 1e-3 & \multicolumn{2}{c}{\texttt{decay\_to\_1e-6}} & \texttt{0,2,500}& \texttt{single}&10\\

     \midrule
     \midrule

     \textit{architecture} & \textit{dataset}& \multicolumn{2}{c}{\textit{data augmentation}} &\textit{warm-up}& \textit{batch size} & \# \textit{epochs} & $\epsilon_{\text{train}}$ & $\epsilon_{\text{test}}$ \\
     \midrule
     \multirow{5}{*}{\shortstack[l]{\mnistmodel\\ \\ GloRo (CE)}} &MNIST& \multicolumn{2}{c}{\texttt{None}} & 0 & 256& 300& 1.75& 1.58\\
     \cmidrule{2-9}\\
     & \textit{initialization} & \textit{init\_lr} & \multicolumn{2}{c}{\textit{lr\_decay}} & \textit{loss}  & $\epsilon$ \textit{schedule}& \textit{power\_iter}\\
     \cmidrule{2-9}\\
     &\texttt{default} & 1e-3 & \multicolumn{2}{c}{\texttt{decay\_to\_5e-6}} & \texttt{CE}& \texttt{single}&10\\

    \midrule
    \midrule

    \textit{architecture} & \textit{dataset} & \multicolumn{2}{c}{\textit{data augmentation}} &\textit{warm-up}& \textit{batch size} & \# \textit{epochs} & $\epsilon_{\text{train}}$ & $\epsilon_{\text{test}}$ \\
    \midrule
    \multirow{5}{*}{\shortstack[l]{\mnistmodel\\ \\ GloRo (T)}} &MNIST & \multicolumn{2}{c}{\texttt{None}} & 0 & 256& 300& 1.75& 1.58 \\
    \cmidrule{2-9}\\
    & \textit{initialization}  & \textit{init\_lr} & \multicolumn{2}{c}{\textit{lr\_decay}} & \textit{loss}  & $\epsilon$ \textit{schedule}& \textit{power\_iter}\\
    \cmidrule{2-9}\\
    &\texttt{default} & 1e-3 & \multicolumn{2}{c}{\texttt{decay\_to\_5e-6}} & \texttt{0,3,300}& \texttt{single}&10\\

    \midrule
    \midrule

    \textit{architecture} & \textit{dataset} & \multicolumn{2}{c}{\textit{data augmentation}} &\textit{warm-up}& \textit{batch size} & \# \textit{epochs} & $\epsilon_{\text{train}}$ & $\epsilon_{\text{test}}$\\
    \midrule
    \multirow{5}{*}{\shortstack[l]{\cifarmodel\\ \\ GloRo (CE)}} &CIFAR-10 & \multicolumn{2}{c}{\texttt{default}} & 20 & 256 &800&0.1551 &0.141 \\
     \cmidrule{2-9}\\
     & \textit{initialization} & \textit{init\_lr} & \multicolumn{2}{c}{\textit{lr\_decay}} & \textit{loss}  & $\epsilon$ \textit{schedule}& \textit{power\_iter}\\
     \cmidrule{2-9}\\
     &\texttt{orthogonal}& 1e-3 & \multicolumn{2}{c}{\texttt{decay\_to\_1e-6}}&\texttt{CE} &\texttt{log} &5\\

    \midrule
    \midrule

    \textit{architecture} & \textit{dataset} & \multicolumn{2}{c}{\textit{data augmentation}} &\textit{warm-up}& \textit{batch size} & \# \textit{epochs} & $\epsilon_{\text{train}}$ & $\epsilon_{\text{test}}$\\
    \midrule
    \multirow{5}{*}{\shortstack[l]{\cifarmodel\\ \\ GloRo (T)}} &CIFAR-10 & \multicolumn{2}{c}{\texttt{default}} &20 & 256 &800&0.1551 &0.141 \\
     \cmidrule{2-9}\\
     & \textit{initialization} & \textit{init\_lr} & \multicolumn{2}{c}{\textit{lr\_decay}} & \textit{loss}  & $\epsilon$ \textit{schedule}& \textit{power\_iter}\\
     \cmidrule{2-9}\\
     &\texttt{default}& 1e-3 & \multicolumn{2}{c}{\texttt{decay\_to\_1e-6}}&\texttt{1.2} &\texttt{log} &5\\

     \midrule
     \midrule

     \textit{architecture} & \textit{dataset} & \multicolumn{2}{c}{\textit{data augmentation}} &\textit{warm-up}& \textit{batch size} & \# \textit{epochs} & $\epsilon_{\text{train}}$ & $\epsilon_{\text{test}}$ \\
     \midrule
     \multirow{5}{*}{\shortstack[l]{\tinyimagenetmodel\\ \\ GloRo (CE)}} &Tiny-Imagenet &\multicolumn{2}{c}{\texttt{default}} & 0 & 256& 250& 0.16& 0.141\\
     \cmidrule{2-9}
      & \textit{optimizer} & \textit{init\_lr} & \multicolumn{2}{c}{\textit{lr\_decay}} & \textit{loss}  & $\epsilon$ \textit{schedule}& \textit{power\_iter}\\
      \cmidrule{2-9}\\
      & \texttt{default} & 2.5e-4 &\multicolumn{2}{c}{\texttt{decay\_to\_5e-7}}  &\texttt{CE} &\texttt{single} & 5\\

     \midrule
     \midrule

     \textit{architecture} & \textit{dataset} & \multicolumn{2}{c}{\textit{data augmentation}} &\textit{warm-up}& \textit{batch size} & \# \textit{epochs} & $\epsilon_{\text{train}}$ & $\epsilon_{\text{test}}$ \\
     \midrule
     \multirow{5}{*}{\shortstack[l]{\tinyimagenetmodel\\ \\ GloRo (T)}} &Tiny-Imagenet &\multicolumn{2}{c}{\texttt{default}} & 0 & 256& 500& 0.16& 0.141\\
      \cmidrule{2-9}\\
      & \textit{initialization} & \textit{init\_lr} & \multicolumn{2}{c}{\textit{lr\_decay}} & \textit{loss}  & $\epsilon$ \textit{schedule}& \textit{power\_iter}\\
      \cmidrule{2-9}\\
      &\texttt{default} & 2.5e-4 & \multicolumn{2}{c}{\texttt{decay\_to\_5e-7}}  & \texttt{1,10,500}&\texttt{single} &1\\
    \bottomrule
    \end{tabular}}
    \end{center}
    \caption{
    Hyperparameters used for training (ReLU) GloRo Nets. We provide models trained with both TRADES (Definition~\ref{def:trades_loss}) loss (denoted by ``(T)'') and with cross-entropy loss (denoted by ``(CE)'').}\label{appendix:fig:supply:hyperparameters}
\end{table*}

In this appendix, we describe hyperparameters used in the training of GloRo Nets to produce the results in Section~\ref{sec:eval}. 
The full set of hyperparameters used for all experiments is shown in Tables~\ref{appendix:fig:supply:hyperparameters-minmax} and \ref{appendix:fig:supply:hyperparameters}. 
We explain each column as follows and discuss how a particular value is selected for each hyperparameter.

\paragraph{Architectures.}\label{appendix:hyperparams:architecture}

To denote architectures, we use $c(C,K,S)$ to denote a convolutional layer with $C$ output channels, a kernel of size $K \times K$, and strides of width $S$. 
We use \texttt{SAME} padding unless noted otherwise. 
We use $d(D)$ to denote a dense layer with $D$ output dimensions. 
We use MinMax~\cite{anil19a} or ReLU activations (see Appendix~\ref{appendix:minmax_vs_relu} for a comparison) after each layer except the top of the network, and do not include an explicit Softmax activation. 
Using this notation, the architectures referenced in Section~\ref{sec:eval} are as shown in the following list.

\begin{itemize}
    \item \mnistmodelsmall: c(16,4,2).c(32,4,2).d(100).d(10)
    \item \mnistmodel: c(32,3,1).c(32,4,2).c(64,3,1).c(64,4,2)\\.d(512).d(512).d(10)
    \item \cifarmodel: c(32,3,1).c(32,3,1).(32,4,2).(64,3,1)\\
    .c(64,3,1).c(64,4,2).d(512).d(10)
    \item \tinyimagenetmodel: c(64,3,1).c(64,3,1).c(64,4,2).c(128,3,1).\\c(128,4,2).c(256,3,1).(256,4,2).d(200)
\end{itemize}

We arrived at these architectures in the following way.
\mnistmodelsmall, \mnistmodel~and \cifarmodel~are used in the prior work~\cite{lee20local_margin, croce19mmr, wong2017provable} to evaluate the verifiable robustness, and we used them to facilitate a direct comparison on training cost and verified accuracy. 
For Tiny-ImageNet, we additionally explored the architecture described in~\cite{lee20local_margin} for use with that dataset, but found that removing one dense and one convolutional layer (denoted by \tinyimagenetmodel~in the list above) produced the same (or better) verified accuracy, but lowered the total training cost.

\paragraph{Data preprocessing.}\label{appendix:hyperparams:data-preprocessing}
For all datasets, we scaled the features to the range [0,1]. 
On some datasets, we used the following data augmentation pipeline \texttt{ImageDataGenerator} from \texttt{tf.keras}, which is denoted by \texttt{default} in Table~\ref{appendix:fig:supply:hyperparameters} and ~\ref{appendix:fig:supply:hyperparameters-minmax}.
\begin{verbatim}
  rotation_range=20
  width_shift_range=0.2
  height_shift_range=0.2
  horizontal_flip=True
  shear_range=0.1
  zoom_range=0.1
\end{verbatim}

When integrating our code with \texttt{tensorflow-dataset}, we use the following augmentaiton pipeline and denote it as \texttt{tfds} in Table~\ref{appendix:fig:supply:hyperparameters} and ~\ref{appendix:fig:supply:hyperparameters-minmax}.

\begin{verbatim}
    horizontal_flip=True
    zoom_range=0.25
    random_brightness=0.2
\end{verbatim}

Our use of augmentation follows the convention established in prior work~\cite{NEURIPS2018_358f9e7b,lee20local_margin}: we only use it on CIFAR and tiny-imatenet, but not on MNIST.

\paragraph{$\epsilon$ scheduling.}\label{appendix:hyperparams:epsilon-schedule}
Prior work has also established a convention of gradually scaling $\epsilon$ up to a value that is potentially larger than the one used to measure verified accuracy leads to better results.
We made use of the following schemes for accomplishing this.
\begin{itemize}
    \item No scheduling: we use `\texttt{single}' to denote we $\epsilon_\text{train}$ for all epochs.
    \item Linear scheduling: we use a string `\texttt{x,y,e}' to denote the strategy that at training epoch $t$, we use $\epsilon_t = x + (y - x)*(t/e)$ if $t \leq e$. When $t>e$, we use the provided $\epsilon_{\text{train}}$ to keep training the model.
    \item Logarithmic scheduling: we use `\texttt{log}' to denote that we increase the epsilon with a logarithmic rate from 0 to $\epsilon_{\text{train}}$.
\end{itemize}
We found that scheduling $\epsilon$ is often unnecessary when instead scheduling the TRADES parameter $\lambda$ (discussed later in this section), which appears to be more effective for that loss.
To select a scheme for scheduling $\epsilon$, we compared the results of the three options listed above, and selected the one that achieved the highest verified accuracy.
If there was no significant difference in this metric, then we instead selected the schedule with the least complexity, assuming the following order: \texttt{single}, (\texttt{x,y,e}), \texttt{log}.
When applying (\texttt{x,y,e}) and \texttt{log}, we began the schedule on the first epoch, and ended it on $(\mathrm{\#\ epochs}) / 2$.

\paragraph{Initialization \& optimization.}\label{appendix:hyperparams:initialization}
In Table~\ref{appendix:fig:supply:hyperparameters}, \texttt{default} refers to the Glorot uniform initialization, given by $\mathtt{tf.keras.initializers.GlorotUniform()}$.
The string `\texttt{ortho}' refers to an orthogonal initialization given by $\mathtt{tf.keras.initializers.Orthogonal()}$.
To select an initialization, we compared the verified accuracy achieved by either, and selected the one with the highest metric.
In the case of a tie, we defaulted to the Glorot uniform initialization.
We used the \texttt{adam} optimizer to perform gradient descent in all experiments, with the initial learning rate specified in Table~\ref{appendix:fig:supply:hyperparameters} and ~\ref{appendix:fig:supply:hyperparameters-minmax}, and default values for the other hyperparameters ($\beta_1=0.9$, $\beta_2=0.999$, $\epsilon=1e-07$, $\mathtt{amsgrad}$ disabled).

\paragraph{Learning rate scheduling.}\label{appendix:hyperparams:lr-schedule}
We write `\texttt{decay\_to\_lb}' to denote a schedule that continuously decays the learning rate to $\mathtt{lb}$ at a negative-exponential rate, starting the decay at $(\mathrm{\# epochs}) / 2$. 
To select $\mathtt{lb}$, we searched over values $\mathtt{lb} \in \{1 \times 10^{-7}, 5 \times 10^{-7}, 1 \times 10^{-6}, 5 \times 10^{-6}\}$, selecting the value that led to the best VRA.
We note that for all datasets except Tiny-Imagenet, we used the default initial rate of $1 \times 10^{-3}$.
On Tiny-Imagenet, we observed that after several epochs at this rate, as well as at $5 \times 10^{-4}$, the loss failed to decrease, so again halved it to arrive at $2.5 \times 10^{-4}$.

\paragraph{Batch size \& epochs.}\label{appendix:hyperparams:batch-size} 
For all experiments, we used minibatches of 256 instances.
Because our method does not impose a significant memory overhead, we found that this batch size made effective use of available hardware resources, increasing training time without impacting verified accuracy, when compared to minibatch sizes 128 and 512.
Because the learning rate, $\epsilon$, and $\lambda$ schedules are all based on the total number of epochs, and can have a significant effect on the verified accuracy, we did not monitor convergence to determine when to stop training.
Instead, we trained for epochs in the range $[100, 1000]$ in increments of 100, and when verified accuracy failed to increase with more epochs, attempted training with fewer epochs (in increments of 50), stopping the search when the verified accuracy began to decrease again.

\paragraph{Warm-up.}\label{appendix:hyperparams:warm-up}
\citet{lee20local_margin} noted improved performance when models were pre-trained for a small number of epochs before optimizing the robust loss.
We found that this helped in some cases with GloRo networks as well, in particular on CIFAR-10, where we used the same number of warm-up epochs as prior work.

\paragraph{$\lambda$ scheduling.}\label{appendix:hyperparams:trades}
When using the TRADES loss described in Section~\ref{sec:implementation}, we found that scheduling $\lambda$ often yielded superior results.
We use `\texttt{x,y,e}' to denote that at epoch $t$, we set $\lambda_t = x + (y-x) * (t/e)$ if $t < e$ else $\lambda_t = y$. 
We write `\texttt{x}' to denote we use $\lambda=x$ all the time.
To select the final $\lambda$, we trained on values in the range $[1,10]$ in increments of $1$, and on finding the whole number that yielded the best result, attempted to further refine it by increasing in increments $0.1$.

\paragraph{Power iteration.}\label{appendix:hyperparams:power-iteration}
As discussed in Section~\ref{sec:implementation}, we use power iteration to compute the spectral norm for each layer to find the layer-wise Lipschitz constants. 
In Table~\ref{appendix:fig:supply:hyperparameters}, \textit{power\_iter} denotes the number of iterations we run for each update during training.
We tried values in the set $\{1,2,5,10\}$, breaking ties to favor fewer iterations for less training cost.
After each epoch, we ran the power iteration until convergence (with tolerance $1 \times 10^{-5}$), and all of the verified accuracy results reported in Section~\ref{sec:eval} are calculated using a global bound based on running power iteration to convergence as well.
Since the random variables used in the power iterations are initialized as \texttt{tf.Variables}, they are stored in \texttt{.h5} files together with the architecture and network parameters. Therefore, one can directly use the converged random variables from the training phase during the test phase.

\paragraph{Search strategy.}
Because of the number of hyperparameters involed in our evaluation, and limited hardware resources, we did not perform a global grid search over all combinations of hyperparameters discussed here.
We plan to do so in future work, as it is possible that results could improve as we may have missed better settings than those explored to produce the numbers reported in our evaluation.
Instead, we adopted a greedy strategy, prioritizing the choices that we believed would have the greatest impact on verified accuracy and training cost.
In general, we explored parameter choices in the following order: $\epsilon$ schedule, $\lambda$ schedule, \# epochs, LR decay, warm-up, initialization, \# power iterations, minibatch size.

\section{Comprehensive VRA Comparisons}
\label{appendix:comprehensive_vra}

\begin{table}
    \centering
    \resizebox{\columnwidth}{!}{%
    \begin{tabular}{ll|cccl}
        \multicolumn{6}{c}{\textbf{Deterministic Guarantees}} 
        \\
        \midrule
        \hspace{2em} &\textit{method} & Clean (\%)& PGD (\%) & VRA(\%)&\hspace{2em} \\
        \midrule
        \multicolumn{6}{c}{\textbf{MNIST} $(\epsilon=0.3)$} \\
        \midrule
        &GloRo & 99.0 &97.8 &95.7 \\
        &BCP & 93.4 &89.5 &84.7\\
        &KW & 98.9 & 97.8&94.0\\
        &MMR  & 98.2 &96.2 &93.6\\
        \midrule
        \multicolumn{6}{c}{\textbf{MNIST} $(\epsilon=1.58)$} \\
        \midrule
        &GloRo & 97.0 & 81.9 &62.8\\
        &BCP & 92.4 &65.8 &47.9\\
        &KW & 88.1 &67.9 &44.5\\
        &BCOP  & 98.8 & - &56.7\\
        &LMT  &  86.5& 53.6 &40.6\\
        \midrule
        \multicolumn{6}{c}{\textbf{CIFAR} $(\epsilon=\nicefrac{36}{255})$} \\
        \midrule
        &GloRo & 77.0 & 69.2 & 58.4\\
        &BCP & 65.7 & 60.8&51.3\\
        &KW & 60.1 & 56.2 & 50.9\\
        &Cayley  &  75.3 & 67.6 &59.1\\
        &BCOP  & 72.4 & 64.4 &58.7\\
        &LMT  & 63.1 & 58.3&38.1\\
        \midrule
        \multicolumn{6}{c}{} \\
        \multicolumn{6}{c}{\textbf{Stochastic Guarantees}} \\
        &\textit{method} & Clean (\%)& PGD (\%) & VRA(\%)\\
        \midrule
        \multicolumn{6}{c}{\textbf{CIFAR} $(\epsilon=0.5)$} \\
        \midrule
        &RS & 67.0 & - &49.0\\
        &SmoothADV & - & - & 63.0\\
        &MACER & 81.0 & - & 59.0\\
        \bottomrule
        \end{tabular}}
        \caption{Comprehensive VRA comparisons for deterministic guarantees and stochastic guarantees. Best results reported in the literature are included in the table.}
        \label{appendix:figure:full-results}
\end{table}

In Section~\ref{sec:eval} of the main paper, we compare, in depth, the performance of GloRo Nets to two approaches to deterministic certification that have been reported as achieving the state-of-the-art in recently published work on robustness certification~\cite{lee20local_margin}.

For completeness, we present a brief overview of a wider range of approaches, providing the VRAs reported in the original respective papers.
Table~\ref{appendix:figure:full-results} contains VRAs reported by several other approaches to deterministic certification, including the methods compared against in Section~\ref{sec:eval}: KW~\cite{wong2017provable} and BCP~\cite{lee20local_margin}; work that was superseded by KW or BCP: MMR~\cite{croce19mmr} and LMT~\cite{tsuzuku18margin}; work that we became aware of after the completion of this work: BCOP~\cite{NEURIPS2019_1ce3e6e3}; and concurrent work: Cayley~\cite{trockman2021orthogonalizing}.
In addition, we include work that provides a \emph{stochastic} guarantee: Randomized Smoothing (RS)~\cite{cohen19smoothing}, SmoothADV~\cite{NEURIPS2019_3a24b25a}, and MACER~\cite{Zhai2020MACER}.
The results for stochastic certification typically use different radii, as reflected in Table~\ref{appendix:figure:full-results}.
We note that because these numbers are taken from the respective papers, the results in Table~\ref{appendix:figure:full-results} should be interpreted as ball-park figures, as they do not standardize the architecture, data scaling and augmentation, etc., and are thus not truly ``apples-to-apples.'' 

We find that GloRo Nets provide the highest VRA for both $\epsilon=0.3$ and $\epsilon=1.58$ on MNIST. 
GloRo Nets also match the result on CIFAR-10 from concurrent work, Cayley, coming within one percentage point of the VRA reported by \citeauthor{trockman2021orthogonalizing}.

\vspace{0.4em}
\section{MinMax vs. ReLU GloRo Nets}
\label{appendix:minmax_vs_relu}

\begin{figure*}[t]
\centering
\begin{subfigure}[t]{0.62\textwidth}
\vspace{1.0em}
\resizebox{\textwidth}{!}{%
\begin{tabular}{l|c|ccc|ccc}
\toprule
\textit{method} & Model & Clean (\%) & PGD (\%)& VRA(\%)  & Sec./epoch & \# Epochs & Mem. (MB) \\
\midrule
\multicolumn{8}{c}{\textbf{MNIST} $(\epsilon=0.3)$} \\
\midrule
ReLU GloRo (CE) & \smallmnistmodel &
    98.4 &96.9 & 94.6 &
    0.7 & 500 & 0.7 \\
ReLU GloRo (T) & \smallmnistmodel &
    98.7 & 97.4& 94.6 &
    0.7 & 500 & 0.7 \\
MinMax GloRo & \smallmnistmodel &
99.0 &97.8 & \textbf{95.7}&
    0.9 & 500 & 0.7 \\
\midrule
\multicolumn{8}{c}{\textbf{MNIST} $(\epsilon=1.58)$} \\
\midrule
ReLU GloRo (CE) & \mnistmodel &
    92.9 & 68.9 & 50.1 &
    2.3 & 300 & 2.2 \\
ReLU GloRo (T) &\mnistmodel &
    92.8 & 67.0 & 51.9 &
    2.0 & 300 & 2.2 \\
MinMax GloRo & \mnistmodel &
97.0 & 81.9 & \textbf{62.8}&
    3.7 & 300 & 2.7 \\
\midrule
\multicolumn{8}{c}{\textbf{CIFAR-10} $(\epsilon=\nicefrac{36}{255})$} \\
\midrule
ReLU GloRo (CE) & \cifarmodel &
    70.7 & 63.8 & 49.3 &
    3.2 & 800 & 2.6 \\
ReLU GloRo (T) & \cifarmodel &
    67.9 & 61.3 & 51.0 &
    3.3 & 800 & 2.6 \\
MinMax GloRo & \cifarmodel &
77.0 & 69.2 & \textbf{58.4}&
    6.9 & 800 & 3.6 \\
\midrule
\multicolumn{8}{c}{\textbf{Tiny-Imagenet} $(\epsilon=\nicefrac{36}{255})$} \\
\midrule
ReLU GloRo (CE) & \tinyimagenetmodel & 31.3
    & 28.2 & 13.2 &
    14.0 & 250 & 7.3  \\
ReLU GloRo (T) & \tinyimagenetmodel &27.4
     & 25.6 & 15.6  &
    13.7 & 500 & 7.3 \\
MinMax GloRo & \tinyimagenetmodel
& 35.5
& 32.3 & \textbf{22.4}&
    40.3 & 800 & 10.4 \\
\bottomrule
\end{tabular}}
\caption{}\label{fig:appendix:train-results}
\end{subfigure}
\hspace{0.1em}
\begin{subfigure}[t]{0.35\textwidth}
\vspace{0.1 em}
\resizebox{\textwidth}{!}{%
\begin{tabular}{l|c|cc}
\toprule
\textit{method} & Model & Time (sec.) & Mem. (MB)\\
\midrule
ReLU GloRo & \cifarmodel & 0.2 & 2.5 \\
MinMax GloRo & \cifarmodel & 0.4 & 1.8 \\
\bottomrule
\end{tabular}
}
\caption{}\label{fig:appendix:cert-results}
\vspace{0.5em}
\resizebox{\textwidth}{!}{%
\begin{tabular}{l|c|cc}
\toprule
\textit{method} \hspace{1em} & global UB & global LB & local LB \\
\midrule
\multicolumn{4}{c}{\textbf{MNIST} $(\epsilon=1.58)$} \\
\midrule
Standard     & $5.4\cdot10^4$ & $1.4\cdot10^2$ & $17.1$ \\
ReLU GloRo   & $3.2$          & $3.0$          & $2.1$ \\
MinMax GloRo & $2.3$          & $1.9$          & $0.8$ \\
\midrule
\multicolumn{4}{c}{\textbf{CIFAR-10} $(\epsilon=\nicefrac{36}{255})$} \\
\midrule
Standard & $1.2\cdot10^7$ & $1.1\cdot10^3$ & $96.2$ \\
ReLU GloRo    & $18.9$         & $11.4$         & $6.2$ \\
MinMax GloRo  & $15.8$         & $11.0$         & $3.7$ \\
\midrule
\multicolumn{4}{c}{\textbf{Tiny-Imagenet} $(\epsilon=\nicefrac{36}{255})$} \\
\midrule
Standard     & $2.2\cdot10^7$ & $3.6\cdot10^2$ & $40.7$ \\
ReLU GloRo   & $7.7$          & $3.3$          & $1.5$ \\
MinMax GloRo & $12.5$         & $5.9$          & $0.8$ \\
\bottomrule
\end{tabular}
}
\caption{}\label{fig:appendix:lipschitz-results}
\end{subfigure}
\caption{
\textbf{(\subref{fig:eval:train-results})}
Certifiable training evaluation results on benchmark datasets. Best results highlighted in bold. 
For ReLU GloRo Nets, we provide models trained with both TRADES (Definition~\ref{def:trades_loss}) and with cross-entropy: ``(T)'' indicates that TRADES loss was used and ``(CE)'' indicates that cross-entropy was used. 
\textbf{(\subref{fig:eval:cert-results})}
Certification timing and memory usage results on CIFAR-10 ($\epsilon=\nicefrac{36}{255}$).
\textbf{(\subref{fig:eval:lipschitz-results})}
Upper and lower bounds on the global and average local Lipschitz constant.
In (\subref{fig:eval:train-results}) and (\subref{fig:eval:cert-results}), peak GPU Memory usage is calculated per-instance by dividing the total measurement by the training or certification batch size.
}
\end{figure*}

Recently, \citet{anil19a} proposed replacing ReLU activations with sorting activations to construct a class of \emph{universal Lipschitz approximators}, that that can approximate any Lipschitz-bounded function over a given domain, and \citet{cohen2019universal} subsequently studied the application to robust training.
We found that these advances in architecture complement our work, improving the VRA performance of GloRo Nets substantially compared to ReLU activations.

The results achieved by GloRo Nets in Figure~\ref{fig:eval:train-results} in Section~\ref{sec:eval} of the main paper are achieved using \emph{MinMax} activations~\cite{anil19a} rather than ReLU activations.
Figure~\ref{fig:appendix:train-results} shows a comparison of the VRA that can be achieved by GloRo Nets using ReLU activations as opposed to MinMax activations.
We see that in each case, the GloRo Nets using MinMax activations outperform those using ReLU activations by a several percentage points.
Nonetheless, the ReLU-based GloRo Nets are still competitive with the VRA performance of KW and BCP.

We observed that MinMax activations result in a slight penalty to training and evaluation cost, as they are slightly slower to compute than ReLU activations.
Figures~\ref{fig:appendix:train-results} and \ref{fig:appendix:cert-results} provide the cost in terms of time and memory incurred by GloRo Nets using each activation function.
We see that the MinMax-based GloRo Nets are slightly slower and more memory-intensive; however, the difference is not particularly significant.

Finally, we compared the Lipschitz bounds obtained on MinMax and ReLU GloRo Nets, presented in Figure~\ref{fig:appendix:lipschitz-results}.
We see that the Lipschitz bounds are fairly similar, in terms of both their magnitude a well as their tightness with respect to the empirical lower bounds.

\section{Measuring Memory Usage}\label{appendix:memory_usage}

In our experiments, we used Tensorflow to train and evaluate standard and GloRo networks, and Pytorch to train and evaluate KW and BCP (since \citet{wong2017provable} and \citet{lee20local_margin} implement their respective methods in Pytorch). 
To measure memory usage, we invoked $\mathtt{tf.contrib.memory\_stats.MaxBytesInUse()}$ at the end of each epoch for standard and GloRo networks, and took the peak active use from $\mathtt{torch.cuda.memory\_summary()}$ at the end of each epoch for KW and BCP.

We note that some differences may arise as a result of differences in memory efficiency between Tensorflow and Pytorch.
In particular, Pytorch enables more control over memory management than does Tensorflow.
In order to mitigate this difference as much as possible, we did not disable gradient tracking when evaluating certification times and memory usage in Pytorch.
While gradient tracking is unnecessary for certification (it is only required for training), Tensorflow does not allow this optimization, so by forgoing it the performance results recorded in Figure~\ref{fig:eval:cert-results} in Section~\ref{sec:eval} are more comparable across frameworks.

In Section~\ref{sec:eval:cost}, we note that Randomized Smoothing (RS) training times have been omitted.
This is because RS essentially acts as a post-processing method on top of a pre-trained model.
In practice the only difference between the training routine to produce a model for RS and standard training is the addition of Gaussian noise (mathcing the noise radius used for smoothing) to the data augmentation;
we assume that this has a negligible impact on training cost.

\section{Optimizing for Lipschitz Lower Bounds}\label{appendix:lower_bounds}

Figure~\ref{fig:eval:lipschitz-results} in Section~\ref{sec:eval} gives empirical lower bounds on the global and average local Lipschitz constants on the models trained in our evaluation.
We use optimization to obtain these lower bounds; further details are provided below.

\paragraph{Global Lower Bounds.}
We use the \emph{margin Lipschitz constant}, $K^*_{ij}$ (Definition~\ref{def:margin_lip} in Appendix~\ref{proof:tighter_bound}), which takes a different value for each pair of classes, $i$ and $j$.
To obtain the lower bound we optimize
\begin{equation*}\label{eq:glob_lower}
    \max_{x_1, x_2}~\max_i\left\{\frac{\left|f_{j_1}(x_1) - f_i(x_1) - \left(f_{j_1}(x_2) - f_i(x_2)\right)\right|}{||x_1 - x_2||}\right\}
\end{equation*}
where $j_1 = F(x_1)$.
Optimization is performed using Keras' default \texttt{adam} optimizer with 7,500 gradient steps.
Both $x_1$ and $x_2$ are initialized to random points in the test set;
we perform this optimization over 100 such initial pairs, and report the maximum value obtained over all initializations.

\paragraph{Local Lower Bounds.}
We use a variant of the \emph{margin Lipschitz constant} (Definition~\ref{def:margin_lip} in Appendix~\ref{proof:tighter_bound}) analogous to the local Lipschitz constant at a point, $x_0$, with radius $\epsilon$.
To obtain this lower bound we optimize
\begin{equation*}\label{eq:loc_lower}
    \max_{x_1, x_2~}~\max_i\left\{\frac{\left|f_{j}(x_1) - f_i(x_1) - \left(f_{j}(x_2) - f_i(x_2)\right)\right|}{||x_1 - x_2||}\right\}
\end{equation*}
$$
    \text{subject to}~~||x_1 - x_0|| \leq \epsilon,~||x_2 - x_0|| \leq \epsilon
$$
where $j = F(x_0)$.
Optimization is performed using Keras' default \texttt{adam} optimizer with 5,000 gradient steps.
After each gradient step, $x_1$ and $x_2$ are projected onto the $\epsilon$-ball centered at $x_0$.
Both $x_1$ and $x_2$ are initialized to random points in the test set, and $x_0$ is a fixed random point in the test set.
We perform this optimization over 100 random choices of $x_0$, and report the mean value.

\paragraph{Discussion.}
In Section~\ref{sec:eval:tightness}, we observe that the global upper bound is fairly tight on the GloRo Net trained on MNIST, but decreasingly so on CIFAR-10 and Tiny-Imagenet.
While this suggests that there is room for improvement in terms of the bounds obtained by GloRo Nets, we make note of two subtleties that may impact these findings.

First, the tightness decreases inversely with the dimensionality of the input.
While it is reasonable to conclude that learning tight GloRo Nets in higher dimensions becomes increasingly difficult, it is worth noting that the optimization process described above also becomes more difficult in higher dimensions, meaning that some of the looseness may be attributable to looseness in the \emph{lower} bound, rather than in the upper bound.

Second, the hyperparameters used may have an effect on the tightness of the Lipschitz bound.
As seen in Appendix~\ref{appendix:hyperparams}, different hyperparameters were used on MNIST, CIFAR-10, and Tiny-Imagenet; some of these differences were selected based on impacting training time, which is of greater concern for larger datasets that naturally take longer to train.
Specifically, we note that fewer power iterations were used for CIFAR-10, and even fewer for Tiny-Imagenet.
While this is good for expediency, and still produces state-of-the-art VRA, we note that tighter bounds may be learned by putting more computation time into training, in the form of more power iterations (for example).
More generally, this speculation suggests that slightly different training strategies, hyperparameters, etc., from the ones used in this work may be sufficient to improve the bounds and the VRA achieved by GloRo Nets.
We conclude that future work should further explore this possibility.

\end{document}